\documentclass[10pt]{article}
\usepackage[preprint]{tmlr}
\usepackage{amsmath,amssymb,amsthm,graphicx,booktabs,enumitem,xcolor}
\usepackage{longtable,array}
\usepackage{hyperref}
\usepackage{url}
\DeclareMathOperator{\sinc}{sinc}
\newcommand{\repourl}{https://github.com/mpoppel/trainable-frequency-reach}
\newcommand{\codedoi}{10.5281/zenodo.23974655}
\newcommand{\coderef}{\url{\repourl}, archived at \href{https://doi.org/\codedoi}{doi:\codedoi}}
\newcommand{\proofref}[1]{\hyperref[prf:#1]{proof in Appendix~\ref*{app:proofs-main}, p.~\pageref*{prf:#1}}}
\newcommand{\appref}[2]{\hyperref[#1]{Appendix~\ref*{#1}, #2}}

\newtheorem{lemma}{Lemma}\newtheorem{proposition}{Proposition}\newtheorem{corollary}{Corollary}
\theoremstyle{definition}\newtheorem{definition}{Definition}
\theoremstyle{remark}\newtheorem{remark}{Remark}
\title{\raggedright How Far Can a Trainable Frequency Move?\\Resolution, Grids and Coupling in Trainable-Frequency Quantum Models}
\author{\raggedright\name \mbox{Michael Poppel$^{1,2}$}\quad \mbox{Markus Baumann$^{1}$}\quad \mbox{Sebastian W\"olckert$^{1}$}\quad
\mbox{Claudia Linnhoff-Popien$^{1}$}\quad \mbox{Jonas Stein$^{1}$}\\
\addr $^{1}$LMU Munich, Munich, Germany\\
\addr $^{2}$Aqarios GmbH, Munich, Germany\\[0.5ex]
{\small\upshape\ttfamily michael.poppel@aqarios.com, markus.baumann@ifi.lmu.de, sebastian.woelckert@ifi.lmu.de,\\
linnhoff@ifi.lmu.de, jonas.stein@ifi.lmu.de}}
\begin{document}
\maketitle
\begin{abstract}
Trainable-frequency quantum models train the prefactors of their encoding gates, so that the circuit's Fourier frequencies can
move onto those of the target and a component at a non-integer frequency can be captured by a single frequency instead of many.
From the unit start used in prior work, with every prefactor equal to $1$, however, training rarely moves a frequency onto a
target more than one resolution step away. We show that the target-driven part of the prefactor gradient is a weighted sum of
slopes of the overlap $\sin u/u$, where $u=(\omega-\nu)T$ is the distance between a model frequency $\nu$ and a target frequency
$\omega$, scaled by the half-width $T$ of the input window, so that one resolution step is $u=\pi$. This gives each frequency a
reach. Under gradient flow, an isolated real frequency whose coefficient adapts optimally reaches its target from within half a
step, for every target at least one step from zero frequency; with a fixed coefficient, it can be stopped by a barrier whose
position depends on the phases. To start every target within this reach, a ternary grid of prefactors places a model frequency
within half a step of each frequency in the band, with a number of encoding gates logarithmic in the band. Circuits with a single frequency, trained with Adam and jointly trained amplitudes rather than under gradient
flow, behave as predicted: inside the first side lobe of the overlap, a frequency starts toward its target exactly when it is in
phase with it in $98\%$ of runs, and only $11$ of $768$ frequencies that started beyond the lobe reached their target. In experiments closer to practice, with several model and target frequencies, the grid start raised the share of
successful fits on synthetic targets from $69\%$ to $84\%$ against the best random start, by reducing failures rather than
improving typical fits. With many target frequencies, the lines of the fixed grid would suffice for a near-perfect least-squares fit, but the circuit
does not reach the coefficients; training the prefactors narrows this shortfall, since one moved frequency replaces many
coefficients, but with $8$ and $24$ target frequencies, training the fixed grid ten times longer fitted better. On real series, training the frequencies improved the held-out $R^2$ of the fixed grid at equal training budget by $0.010$--$0.045$ on four
of seven series, but a classical screen did not separate these from the others. A classical model with independently trainable
frequencies shows the same dependence on the start; our results describe how to train such models, not an advantage of the
circuit.
\end{abstract}
\section{Introduction}
\label{sec:intro}

Angle-encoded variational quantum circuits compute truncated Fourier series~\citep{Schuld_2021}: the encoding gates fix
which frequencies the model can use, and the trainable ansatz sets their coefficients. A fixed integer spectrum represents a
target component at a non-integer frequency only approximately, spread over all frequencies of the spectrum with weights that
decay only inversely with their distance from it, and exactly only in the limit of infinitely many terms. In \emph{trainable-frequency} models the prefactors of the encoding gates are trained as
well~\citep{Jaderberg_2024}, so that a single model frequency can move onto such a component instead. This is not reliable:
in a setting following that of~\citet{Jaderberg_2024}, with all prefactors started at one, target frequencies $1.3$ to $2$
resolution steps from every starting frequency were reached in $0$--$12\%$ of runs (Section~\ref{sec:single-init}). In
this paper we examine under which conditions training the frequencies yields reliable results, and what it adds when a target
contains several non-integer frequencies and on real data.

We approach the question through the gradient of the loss with respect to the prefactors, with tools from signal
processing, where a spectrum of discrete frequencies is a line spectrum and recovering its components is line spectral
estimation; we therefore call each model frequency, with its coefficient, a \emph{line}. On an input window $[-T,T]$, the
overlap of a model frequency $\nu$ and a target frequency $\omega$, the inner product of the waves $e^{i\omega x}$ and $e^{i\nu x}$
normalised by the window length, is $\sinc(u)=\sin u/u$ with $u=(\omega-\nu)T$. It equals one when the two frequencies coincide
and first vanishes when they are one resolution step $\pi/T$ apart. Our contributions are:
\begin{itemize}[leftmargin=1.2em,itemsep=2pt,topsep=3pt]
\item \emph{The prefactor gradient in terms of overlaps} (Sections~\ref{sec:theory-exact} and~\ref{sec:theory-self}). The part of
  the prefactor gradient driven by the target is exactly a weighted sum of the slopes $\sinc'(u)$, one for each model line. The
  rest, the self term, is built from the same slopes between pairs of model lines, so that lines which do not overlap still push
  each other.
\item \emph{The reach of a line} (Section~\ref{sec:theory-isolated}). Under gradient flow, an isolated complex line whose amplitude
  adapts optimally is pulled onto its target from within one step, while with a fixed amplitude it moves toward or away from the
  target depending on its phase and, under a condition on the phase, cannot get in from beyond the first side lobe, as in classical
  frequency estimation, where a local search needs a start close to the true frequency (Section~\ref{sec:related}). For a real line (a line with its mirror line) and a constant term, as in a circuit,
  we prove a reach of half a step for every target at least one step from zero frequency, with a
  computer-assisted step certified by interval arithmetic, and a barrier whose position depends on the phases. The proofs are
  formalised in Lean~4.
\item \emph{A start within reach} (Section~\ref{sec:theory-grid}). A ternary grid of prefactors places a model line within half a
  step of every target in the band, with a number of encoding gates logarithmic in the band. On this grid each target also has a
  second model line within one step, so the theory guarantees the start distance, not the training that follows.
\item \emph{Experiments from one line to real data} (Sections~\ref{sec:single}--\ref{sec:realdata}). Single-frequency circuits
  trained with Adam follow the predicted direction and rarely cross the barrier. On targets with three lines, the grid start
  reduces failed fits relative to random starts (success $69\%\to84\%$), and an untrained grid with half the spacing, which takes one
  more encoding gate, fits as well as the trained grid. With many target lines, the fit is limited by reaching the coefficients, not by the lines available:
  least squares on the lines of the fixed grid fits almost perfectly, while the circuit falls short. Training the prefactors
  narrows this shortfall but places few lines, with many target lines hardly more than a random prefactor change of the same size would, and with $8$ and $24$
  target lines, training the fixed grid ten times longer fitted better. On real series, training the frequencies improved the held-out
  $R^2$ of the fixed grid at equal training budget by $0.010$--$0.045$ on four of seven series, but a classical screen did not separate these from the others.
\end{itemize}
A classical model with independently trainable frequencies shows the same dependence on the start, and the results describe how
to train such models rather than an advantage of the circuit.

Section~\ref{sec:related} places this work in the literature, Section~\ref{sec:theory} gives the theory, Sections~\ref{sec:single}--\ref{sec:realdata} the experiments.
Inputs are one-dimensional, and circuits are small and simulated exactly; shot noise is studied only for the final approach of a
single frequency.

\section{Related work}
\label{sec:related}

\paragraph{Fourier view of quantum models.} Angle-encoded circuits compute truncated Fourier series~\citep{Schuld_2021}. Prior
work characterises which functions and frequencies a circuit can
express~\citep{PhysRevA.104.012405,yu2022power,holzer2024spectralinvariancemaximalityproperties}, how its Fourier coefficients depend
on the ansatz parameters~\citep{Wiedmann_2024,mhiri2025constrained}, how the spectrum relates to training, through spectral
bias~\citep{Duffy_2025} and gradient magnitudes~\citep{Barthe_2024}, and why designing a model's spectrum may suit quantum
computers~\citep{Belis_2026}; barren plateaus~\citep{McClean_2018,Larocca_2025} are a
different limit of trainability. \citet{mhiri2025constrained} show that the variance of a coefficient over the ansatz parameters is
linear in the redundancy of its frequency, so the lines of a single layer of ternary encoding gates, whose redundancy is $2^z$ for $z$ zero digits, differ in how
widely their coefficients vary.

\paragraph{Encodings and trainable frequencies.} Encodings with exponentially growing prefactors produce a spectrum that
grows exponentially with the number of encoding gates~\citep{Shin_2023,Kordzanganeh_2023}; ternary prefactors
$(1,3,9,\dots)$ give the largest number of distinct frequencies possible for a separable encoding
Hamiltonian~\citep{Shin_2023,Peters_2023}. \citet{Jaderberg_2024} train the prefactors themselves, so
that the frequencies can move off the integers, and start them at one. How far training can move the prefactors, how to start
them, and what it costs when frequencies share prefactors has, to our knowledge, not been studied.

\paragraph{Line spectral estimation and super-resolution.} Recovering a few frequencies from samples is classical. For a single
line, the maximum-likelihood frequency is the highest peak of the periodogram, but the likelihood is multimodal with a very sharp
global maximum, so a local search needs a very accurate start~\citep[§4.3]{stoica2005spectral}; practical estimators take it
from a coarse search over the discrete Fourier transform and then refine it~\citep{rife1974single}. For several lines separated
by more than the resolution limit of the periodogram, its largest peaks provide such starts~\citep[§4.3]{stoica2005spectral}. For several lines (\emph{spikes} in
this literature), \citet{candes2014towards} show that a convex total-variation formulation recovers them exactly from noiseless
low-frequency measurements if they are separated by at least twice the resolution limit. For gradient descent on positions and
amplitudes jointly, \citet{traonmilin2020basins} give an explicit basin of attraction around the global minimiser, whose radius is
at most a quarter of the minimum separation of the spikes, and \citet{chizat2022sparse} shows that gradient descent on many
particles converges to the global minimiser of a sparse optimisation problem over measures if the particles start close to a
density that covers the domain, under non-degeneracy assumptions. Our results for an isolated line
(Propositions~\ref{prop:isolated} and~\ref{prop:real}) are of the same kind, a basin about one resolution element wide, and the
grid start plays the role of the coarse search and of the dense initialisation. What differs in trainable-frequency circuits is
that the lines cannot move independently: $3^k$ lines are positioned by $k$ prefactors, a real model carries mirror lines and a
constant, and the amplitudes are set by an ansatz rather than being free parameters.

\paragraph{Classical networks with sinusoidal features.} Classical networks with sinusoidal features also depend on how their
frequencies are initialised. Fourier feature networks draw random frequencies at a scale tuned per task to counteract
spectral bias~\citep{tancik2020fourier}; when \citet{tancik2020fourier} trained these frequencies jointly with the network,
performance did not improve and the frequencies stayed close to their initial values (their Appendix~A.3), as the prefactors do
from the unit start in our circuits (Section~\ref{sec:single-init}). Sinusoidal representation networks choose the scale
of their first-layer frequencies and initialise all layers so that the distribution of activations is
preserved~\citep{sitzmann2020implicit}.

\section{Model and problem setting}
\label{sec:setting}

\paragraph{Model and accessible frequencies.}
We consider circuits on $n$ qubits that alternate trainable unitaries with layers of encoding gates
$e^{-i\alpha_j x\sigma_j/2}$, where $x\in\mathbb R$ is the input, $\sigma_j$ a Pauli operator and $\alpha_j$ a trainable
\emph{prefactor}; with $E$ encoding gates in total, $\boldsymbol\alpha\in\mathbb R^E$. The output is the expectation of an
observable $M$. Write $x_j=\alpha_jx$ for the angle of gate $j$. The generator $\sigma_j/2$ has eigenvalues $\pm\frac12$, so in
its eigenbasis the gate contributes phases $e^{\mp ix_j/2}$ to the state and their conjugates to $\langle M\rangle$. The output is
therefore a trigonometric polynomial in $x_1,\dots,x_E$ with frequencies in $\{-1,0,1\}$ for each variable, the differences of two
eigenvalues~\citep{Schuld_2021}. Setting $x_j=\alpha_jx$ gives
\begin{equation}
  f(x)=\sum_{\mathbf t\in\{-1,0,1\}^E} c_{\mathbf t}(\theta)\,e^{i(\mathbf t\cdot\boldsymbol\alpha)x},
  \qquad c_{-\mathbf t}=\overline{c_{\mathbf t}},
  \label{eq:model}
\end{equation}
whose coefficients are those of the polynomial and so depend on the ansatz parameters $\theta$ but not on $\boldsymbol\alpha$.
With all $\alpha_j=1$ this is the integer spectrum of~\citet{Schuld_2021}; with trainable $\alpha_j$ it is the model
of~\citet{Jaderberg_2024}.

\begin{definition}[Line, amplitude, digits]\label{def:line}
The term of Eq.~\eqref{eq:model} with index $\mathbf t\in\{-1,0,1\}^E$ is a \emph{line}; its frequency is
$\nu_{\mathbf t}=\mathbf t\cdot\boldsymbol\alpha$, its coefficient $c_{\mathbf t}$ is its \emph{amplitude}, and $\mathbf t$ are its \emph{digits}.
The line with $\mathbf t=\mathbf 0$ is the \emph{constant term}, and the line with digits $-\mathbf t$ is the \emph{mirror line} of the line
with digits $\mathbf t$.
\end{definition}

\noindent The name follows signal processing, where a spectrum of discrete frequencies is a line spectrum. The digit $t_j$ is
the eigenvalue difference contributed by gate $j$: it states whether the line adds the prefactor $\alpha_j$, subtracts it or does
not contain it. For ternary prefactors the digits are the balanced-ternary digits of the frequency. Lines are indexed by their digits rather than by their frequency because a line moves with the prefactors as a
unit, whereas different digits that give the same frequency at the start move apart once the prefactors change; while they
coincide, their coefficients add. With
$\boldsymbol\alpha=(1,3)$, for example, the digits $(+1,-1)$ give the line $1-3=-2$ (Figure~\ref{fig:schematic}b). Ternary
prefactors $(g,3g,9g,\dots)$ produce the grid of lines $mg$ with spacing $g$, each from one digit vector
(Corollary~\ref{cor:grid}); such exponential encodings were introduced by~\citet{Shin_2023}.

\paragraph{Prefactors act on several lines.}
Differentiating Eq.~\eqref{eq:model},
\begin{equation}
  \partial_{\alpha_j}f(x)=ix\,G_j(x),\qquad G_j(x)=\sum_{\mathbf t}t_j\,c_{\mathbf t}\,e^{i(\mathbf t\cdot\boldsymbol\alpha)x}.
  \label{eq:dalpha}
\end{equation}
Since $\nu_{\mathbf t}=\mathbf t\cdot\boldsymbol\alpha$ is linear in $\boldsymbol\alpha$, a change $\Delta\boldsymbol\alpha$
moves line $\mathbf t$ by exactly $\mathbf t\cdot\Delta\boldsymbol\alpha=\sum_jt_j\Delta\alpha_j$: prefactor $\alpha_j$ moves
every line whose digit $t_j$ is nonzero, in the direction of that digit's sign (Figure~\ref{fig:schematic}b). Training can
therefore move lines independently only where their digit patterns allow it (Section~\ref{sec:coupling}).

\paragraph{Loss, resolution step and gaps.}
In Sections~\ref{sec:theory}--\ref{sec:coupling} the target is a finite sum of lines,
$h(x)=\sum_{\omega\in\Omega_h}h_\omega e^{i\omega x}$ with $\Omega_h=-\Omega_h$. Its terms are \emph{target lines}, with frequency
$\omega$ and amplitude $h_\omega$, and the lines of Definition~\ref{def:line} are \emph{model lines}; where no confusion arises, we
refer to a line by its frequency. The real series of Section~\ref{sec:realdata}
need not be; there a screen decides empirically whether they contain off-grid lines. On the input window $[-T,T]$ we analyse the population loss
\begin{equation}
  \mathcal L=\frac{1}{4T}\int_{-T}^{T}\bigl(f(x)-h(x)\bigr)^2\,dx,\qquad
  \partial_{\alpha_j}\mathcal L
  =\underbrace{-\frac{1}{2T}\int_{-T}^{T}h\,\partial_{\alpha_j}f\,dx}_{\mathcal T_{\mathrm{tgt}}}
  \;+\;\underbrace{\frac{1}{2T}\int_{-T}^{T}f\,\partial_{\alpha_j}f\,dx}_{\mathcal T_{\mathrm{self}}},
  \label{eq:loss}
\end{equation}
and train on its empirical counterpart, the mean squared error over sampled inputs. The \emph{target-driven} term
$\mathcal T_{\mathrm{tgt}}$ is the part of the prefactor gradient that depends on the target; $\mathcal T_{\mathrm{self}}$ depends
on the model only. Both terms are built from the following inner product.

\begin{definition}[Overlap, resolution step, gap]\label{def:overlap}
The \emph{overlap} of two lines with frequencies $\nu$ and $\omega$ on the window $[-T,T]$ is
$\frac1{2T}\int_{-T}^{T}e^{i(\omega-\nu)x}dx=\sinc((\omega-\nu)T)$, with $\sinc(z)=\sin z/z$ and $\sinc(0)=1$. The \emph{resolution step}
is $\pi/T$, the smallest distance at which the overlap vanishes. The \emph{gap} between a model line $\nu$ and a target line $\omega$ is
\begin{equation}
  u=(\omega-\nu)\,T ,
  \label{eq:gap}
\end{equation}
so that a gap of one step is $u=\pi$.
\end{definition}

\noindent The factor $x$ in Eq.~\eqref{eq:dalpha} turns $\mathcal T_{\mathrm{tgt}}$ and $\mathcal T_{\mathrm{self}}$ into sums of slopes of
the overlap (Section~\ref{sec:theory}). All distances between frequencies are given in resolution steps. Any finite input range
can be mapped to $[-\pi,\pi]$ by an affine rescaling; there one step is one frequency unit.

\begin{figure}[t]\centering
\includegraphics[width=\linewidth]{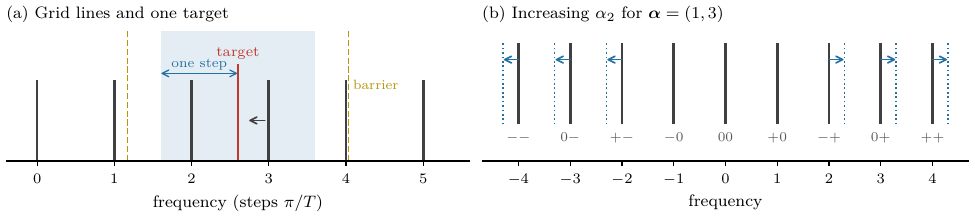}
\caption{(a) Grid lines one resolution step $\pi/T$ apart and one target. An isolated line with best-response amplitude is
pulled onto the target from within one step (shaded); with a fixed amplitude a line that starts beyond the first side lobe of the overlap
never gets closer than that side lobe (dashed). Both statements are for a complex line (Proposition~\ref{prop:isolated}); for a real line the
reach holds from half a step and the barrier moves with the phases (Proposition~\ref{prop:real}). (b) Lines $\nu=\mathbf t\cdot\boldsymbol\alpha$ for $\boldsymbol\alpha=(1,3)$, labelled by the signs of their digits
$(t_1,t_2)$, e.g.\ $+-$ for $1-3=-2$. Increasing $\alpha_2$ by $\Delta$ moves every line with $t_2\neq0$ by $\Delta$ in the direction of that digit's
sign (arrows), for any prefactors; a change of several prefactors moves line $\mathbf t$ by $\mathbf t\cdot\Delta\boldsymbol\alpha$. Both panels show positions only; the heights of the lines are schematic, since the amplitudes are set by the
ansatz.}
\label{fig:schematic}
\end{figure}
\section{Gradient structure and reach of a trainable frequency}
\label{sec:theory}

Sections~\ref{sec:theory-exact} and~\ref{sec:theory-self} give the size of each contribution to the prefactor gradient,
Section~\ref{sec:theory-isolated} the direction of motion of a single line, and Section~\ref{sec:theory-grid} a start within its
reach. All proofs are in Appendix~\ref{app:proofs}.

Throughout, $T>0$ is the half-width of the input window and primes denote derivatives with respect to the argument. A
\emph{solution} of a gradient flow $\dot\nu=-\partial_\nu\mathcal E(\nu)$ is a function $\nu:[0,\infty)\to\mathbb R$, differentiable for
$\tau\ge0$, that satisfies the equation for all $\tau\ge0$. For the flows below, the vector field is continuously differentiable wherever
the solutions considered can go (for the best-response flow of a real line, away from zero frequency, which part~(a) of
Proposition~\ref{prop:real} keeps the line from), so by the Picard--Lindel\"of theorem a solution exists for all $\tau\ge0$ and
is unique; the statements below, and their formal proofs, concern a given solution.

\subsection{The target-driven term as a sum of overlap slopes}
\label{sec:theory-exact}

By Eq.~\eqref{eq:dalpha}, the prefactor derivative $\partial_{\alpha_j}f=ix\,G_j(x)$ is a sum over the model lines, with
frequencies $\nu_{\mathbf t}=\mathbf t\cdot\boldsymbol\alpha$ and weights $t_jc_{\mathbf t}$. Inserting it and the target lines, with
frequencies $\omega\in\Omega_h$, into
$\mathcal T_{\mathrm{tgt}}$ of Eq.~\eqref{eq:loss} produces, for each pair of a target line and a model line, the window average of
$x\,e^{i(\omega+\nu_{\mathbf t})x}$. The following lemma evaluates this average.

\begin{lemma}[Overlap slope]\label{lem:kernel}
For every $T>0$ and every $n\in\mathbb R$,
\[
  \frac{1}{2T}\int_{-T}^{T}x\,e^{inx}\,dx=-\,i\,T\,\sinc'(nT),\qquad
  \sinc'(z)=\frac{z\cos z-\sin z}{z^2}\ (z\neq0),\quad \sinc'(0)=0 .
\]
\end{lemma}

\noindent\emph{Proof sketch.} The left-hand side is $-i\,\frac{d}{dn}\frac1{2T}\int_{-T}^{T}e^{inx}dx=-i\,\frac{d}{dn}\sinc(nT)$; the derivative of
the overlap is evaluated by the chain rule (\proofref{kernel}).

\begin{proposition}[Exact target-driven term]\label{prop:exact}
Let $f$ be of the form of Eq.~\eqref{eq:model} and $h(x)=\sum_{\omega\in\Omega_h}h_\omega e^{i\omega x}$. For every $j\in\{1,\dots,E\}$,
every $\theta$ and every $\boldsymbol\alpha\in\mathbb R^E$,
\[
  \mathcal T_{\mathrm{tgt}}=\sum_{\omega\in\Omega_h}\mathcal T_{\mathrm{tgt}}^{(\omega)},\qquad
  \mathcal T_{\mathrm{tgt}}^{(\omega)}=-\,T\,h_\omega\sum_{\mathbf t\in\{-1,0,1\}^E}t_j\,c_{\mathbf t}\,\sinc'\!\bigl((\omega+\nu_{\mathbf t})T\bigr).
\]
\end{proposition}

\noindent\emph{Proof sketch.} Expand $h$ and $G_j$ in $\mathcal T_{\mathrm{tgt}}=-\frac1{2T}\int h\,\partial_{\alpha_j}f\,dx$ and apply
Lemma~\ref{lem:kernel} to each pair of a target line and a model line (\proofref{exact}).

The weight $h_\omega t_jc_{\mathbf t}$ of each pair is determined by the target and the ansatz parameters, and the slope
$\sinc'$ only by the distance between the two lines. The argument is $\omega+\nu_{\mathbf t}$ rather than $\omega-\nu_{\mathbf t}$ because
no complex conjugate enters $\mathcal T_{\mathrm{tgt}}$; since the digits $-\mathbf t$ give the mirror line $-\nu_{\mathbf t}$, every line
$\nu$ contributes the slope $\sinc'((\omega-\nu)T)=\sinc'(u)$ at the gap $u$ of Eq.~\eqref{eq:gap}.

\begin{corollary}[Size of each contribution]\label{cor:bound}
For $T>0$ and $n\neq0$,
\[
  T\,|\sinc'(nT)|\;\le\;K_T(n):=\min\Bigl\{\frac{T^2|n|}{3},\;\frac{2}{|n|}\Bigr\}.
\]
Consequently the pair of a target line $\omega$ and a model line $\nu_{\mathbf t}$ contributes at most
$|h_\omega t_jc_{\mathbf t}|\,K_T(\omega+\nu_{\mathbf t})$ to $|\mathcal T_{\mathrm{tgt}}^{(\omega)}|$, and if $|\omega+\nu_{\mathbf t}|\ge\delta>0$
for every $\mathbf t$, then $|\mathcal T_{\mathrm{tgt}}^{(\omega)}|\le(2|h_\omega|/\delta)\sum_{\mathbf t}|t_jc_{\mathbf t}|$.
\end{corollary}

\noindent\emph{Proof sketch.} From $\sinc(z)=\int_0^1\cos(sz)\,ds$ one obtains $|\sinc'(z)|\le|z|/3$, and from
$z\sinc'(z)=\cos z-\sinc(z)$ one obtains $|z\sinc'(z)|\le2$ (\proofref{bound}).

The first bound vanishes linearly as the two lines approach each other, so the contribution of a pair vanishes when its model
line coincides with its target line; the contributions of the other pairs, including the target's mirror line, do not. The second decays inversely with the distance; it bounds each pair separately, and the sum over many model lines with
weights of equal sign can grow logarithmically with the band. The corollary bounds magnitudes only: the signs of the weights
depend on $\theta$, so it determines neither the sign of $\mathcal T_{\mathrm{tgt}}$ nor the direction in which a line moves.
Section~\ref{sec:theory-isolated} determines that direction for a single line.

\subsection{The self term}
\label{sec:theory-self}

The remaining part of the prefactor gradient, $\mathcal T_{\mathrm{self}}=\frac1{2T}\int_{-T}^{T}f\,\partial_{\alpha_j}f\,dx$, has the same
structure.

\begin{corollary}[Self term]\label{cor:self}
Under the assumptions of Proposition~\ref{prop:exact},
\[
  \mathcal T_{\mathrm{self}}=T\sum_{\mathbf t,\mathbf t'\in\{-1,0,1\}^E}c_{\mathbf t'}\,t_j\,c_{\mathbf t}\,\sinc'\!\bigl((\nu_{\mathbf t}+\nu_{\mathbf t'})T\bigr),
\]
and each term obeys the bounds of Corollary~\ref{cor:bound}.
\end{corollary}

\noindent\emph{Proof sketch.} Proposition~\ref{prop:exact} with $f$ in place of $h$ and the opposite sign (\proofref{self}).

A line paired with its own mirror ($\mathbf t'=-\mathbf t$) contributes $\sinc'(0)=0$; every other pair contributes. A line
therefore interacts with its mirror line through $\sinc'(2\nu T)$ and with the constant term through $\sinc'(\nu T)$
(Proposition~\ref{prop:real} accounts for both), and in a model with several lines each line interacts with all others. On a grid
with a spacing of one resolution step the lines are orthogonal, $\sinc(k\pi)=0$, whereas their overlap slopes are not:
$\sinc'(k\pi)=(-1)^k/(k\pi)$. Orthogonal model lines thus contribute to each other's gradient, with a magnitude that decays inversely
with their separation. At an exact fit, $f=h$, the full gradient vanishes, so the self term and the target-driven term cancel there;
a self term of the same magnitude as the target-driven term near a fit is therefore expected. Section~\ref{sec:coupling} measures
its magnitude at the start of joint training.

\subsection{A single line under gradient flow}
\label{sec:theory-isolated}

To determine the direction of motion we consider one model line with complex amplitude $c$, set by the ansatz, against one
target line of unit amplitude:
\[
  \ell(c,\nu)=\frac{1}{2T}\int_{-T}^{T}\bigl|c\,e^{i\nu x}-e^{i\omega x}\bigr|^2dx ,\qquad u=(\omega-\nu)T .
\]
For a line with digits $\mathbf t$, gradient descent on the prefactors gives $\dot\nu=-|\mathbf t|^2\partial_\nu\ell$ when no other line
shares its prefactors; the factor $|\mathbf t|^2$ rescales time and is omitted. Training changes $\nu$ and $c$ jointly. We analyse
the two limits of this coupled system.

\begin{definition}[Best response, phase, side lobes]\label{def:br}
For fixed $\nu$, the \emph{best response} is the amplitude that minimises $\ell(\cdot,\nu)$, and $\ell^\ast(\nu)$ is the minimum. For a fixed
amplitude $c$, the line is \emph{in phase} with the target if $\operatorname{Re}c>0$ and \emph{out of phase} if $\operatorname{Re}c<0$. For
$k\ge1$, $u_k$ denotes the unique zero of $\sinc'$ in $(k\pi,(k+1)\pi)$; the interval around $u_k$ between consecutive zeros of
$\sinc$ is the $k$-th \emph{side lobe} of the overlap.
\end{definition}

\begin{proposition}[Reach of a complex line]\label{prop:isolated}
Let $T>0$, $\omega\in\mathbb R$ and $r=\operatorname{Re}c$. Then $\ell(c,\nu)=|c|^2-2r\,\sinc(u)+1$. Moreover:
\begin{enumerate}[label=(\alph*),leftmargin=1.6em,itemsep=2pt]
\item \textbf{Best response.} For fixed $\nu$, $\ell(\cdot,\nu)$ is minimised over $c\in\mathbb C$ exactly at $c=\sinc(u)$, and
  $\ell^\ast(\nu)=1-\sinc(u)^2$. For a solution of $\dot\nu=-\partial_\nu\ell^\ast(\nu)$, the gap satisfies
  $\dot u=2T^2\sinc(u)\sinc'(u)$, and
  \begin{enumerate}[label=(\roman*),leftmargin=1.4em,itemsep=0pt]
  \item if $0<|u(0)|<\pi$, then $u(\tau)\to0$ as $\tau\to\infty$;
  \item if $k\pi<|u(0)|<(k+1)\pi$ for some integer $k\ge1$, then $|u(\tau)|\to u_k$, where $u_k\in(k\pi,k\pi+\pi/2)$ and
    $\tan u_k=u_k$.
  \end{enumerate}
\item \textbf{Fixed amplitude.} For fixed $c$ and a solution of $\dot\nu=-\partial_\nu\ell(c,\nu)$, the gap satisfies
  $\dot u=2rT^2\sinc'(u)$, and $u_1\in(4.49,4.50)$. Moreover:
  \begin{enumerate}[label=(\roman*),leftmargin=1.4em,itemsep=0pt]
  \item for $0<|u|<u_1$, $\operatorname{sign}\dot u=-\operatorname{sign}(r\,u)$;
  \item if $0<|u(0)|<u_1$, then $u(\tau)\to0$ for $r>0$ and $|u(\tau)|\to u_1$ for $r<0$;
  \item if $|u(0)|>u_1$, then $|u(\tau)|\ge u_1$ for all $\tau\ge0$ and every $r$.
  \end{enumerate}
\end{enumerate}
\end{proposition}

\noindent\emph{Proof sketch.} The loss follows from the overlap of the two lines; completing the square in $c$ gives the best
response. In both limits the gap obeys an autonomous one-dimensional equation $\dot u=F(u)$, whose solutions converge to the
nearest zero of $F$ in the direction of its sign and cannot cross a zero of $F$; the sign pattern of $\sinc'$ (zeros at $u_k$,
negative on $(0,u_1)$) then gives each statement (\proofref{isolated}).

Part~(a) gives a reach of one resolution step: from farther away, a line with best-response amplitude stops at a side lobe, at
$u_k/\pi=1.430$, $2.459$, $3.471,\dots$ steps (the loss $1-\sinc(u)^2$ is shown in Figure~\ref{fig:e0}a). Part~(b) gives a barrier at
the first side lobe for a fixed amplitude.

\begin{figure}[t]\centering
\includegraphics[width=\linewidth]{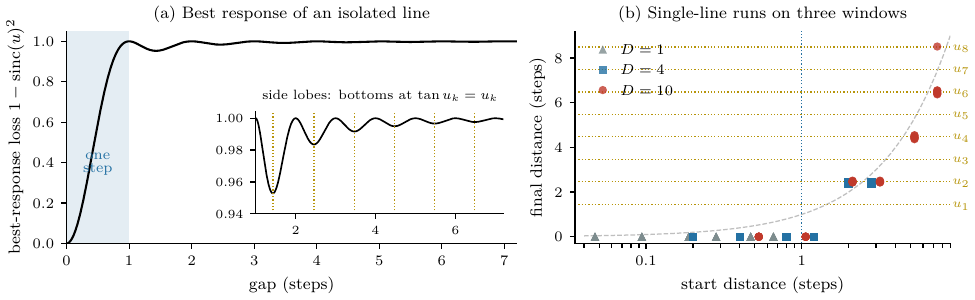}
\caption{(a) Best-response loss $1-\sinc(u)^2$ against the gap in steps; inset: the side lobes, with bottoms at $\tan u_k=u_k$
(dotted). (b) Final against start distance, both in steps, of every single-line run of Section~\ref{sec:single-line} on windows of $D=1$, $4$ and $10$
periods (eight seeds per point; start distance on a log scale). Dashed: no movement; dotted: one step and the side lobes $u_k/\pi$.
Final distances in steps agree across windows, and lines that stall sit on the side lobes (target phase zero, the same eight
ansatz seeds on every window).}
\label{fig:e0}
\end{figure}

\paragraph{Real lines.} The circuit's output and the target are real, so each line has a mirror line at the negative frequency
($c_{-\mathbf t}=\overline{c_{\mathbf t}}$ in Eq.~\eqref{eq:model}), and the line with all digits zero is a constant term. We consider the
model $d+A\cos(\nu x+\varphi)$, formed by a line, its mirror line and the constant term ($A=2|c_{\mathbf t}|\ge0$,
$\varphi=\arg c_{\mathbf t}$, $d=c_{\mathbf 0}$), and call the pair of a line and its mirror line a \emph{real line}; the target
is $\cos(\omega x+\psi)$; a target amplitude
$H>0$ rescales $A$ and $d$ by $1/H$. We write $\bar\omega=\omega T$, so the target lies $\bar\omega/\pi$ steps from zero frequency,
$\kappa=\cos(\varphi-\psi)$ for the phase factor (the line is in phase if $\kappa>0$), and $\gamma(z)=(z+1)/z^2$, which bounds
$|\sinc'(z)|$ for $z>0$.

\begin{proposition}[Reach of a real line]\label{prop:real}
Let $T>0$, $\omega>0$, $\psi\in\mathbb R$, and
\[
  \ell(\nu;d,A,\varphi)=\frac1{2T}\int_{-T}^{T}\bigl(d+A\cos(\nu x+\varphi)-\cos(\omega x+\psi)\bigr)^2dx,\qquad
  \ell^\ast(\nu)=\min_{d\in\mathbb R,\ A\ge0,\ \varphi\in\mathbb R}\ell(\nu;d,A,\varphi).
\]
\begin{enumerate}[label=(\alph*),leftmargin=1.6em,itemsep=2pt]
\item \textbf{Best response.} Let $\bar\omega\ge\pi$, and let $\nu$ be a solution of $\dot\nu=-\partial_\nu\ell^\ast(\nu)$ with
  $|u(0)|\le\pi/2$. Then $|u(\tau)|\le\pi/2$ for all $\tau\ge0$, and $u(\tau)\to0$ as $\tau\to\infty$.
\item \textbf{Fixed amplitude.} Let $d$, $A\ge0$ and $\varphi$ be fixed, $\bar\omega>5$, and let $\nu$ be a solution of
  $\dot\nu=-\partial_\nu\ell(\nu;d,A,\varphi)$. Then $\dot u=T^2A\bigl(\kappa\,\sinc'(u)+e(u)\bigr)$ with
  \[
    |e(u)|\le\delta_{\mathrm{fix}}:=\gamma(2\bar\omega-5)+A\,\gamma(2\bar\omega-10)+2|d|\,\gamma(\bar\omega-5)\qquad\text{for }|u|\le5 .
  \]
  \begin{enumerate}[label=(\roman*),leftmargin=1.4em,itemsep=0pt]
  \item If $|u|\le5$ and $|\kappa\,\sinc'(u)|>\delta_{\mathrm{fix}}$, then $\dot u$ has the sign of $A\kappa\,\sinc'(u)$; in particular, for
    $0<|u|<u_1$, an in-phase line moves toward the target and an out-of-phase line away from it.
  \item If $0<u_-<u_+\le5$ with $\sinc'(u_-)<0<\sinc'(u_+)$ and $|\kappa|\min\{|\sinc'(u_-)|,\sinc'(u_+)\}>\delta_{\mathrm{fix}}$, then
    $|u(0)|\ge u_+$ implies $|u(\tau)|\ge u_-$ for all $\tau\ge0$.
  \end{enumerate}
\end{enumerate}
\end{proposition}

\noindent\emph{Proof sketch.} The loss is a closed-form expression in the overlaps $\sinc(u)$ (target), $\sinc(2\bar\omega-u)$
(target mirror), $\sinc(2\bar\omega-2u)$ (own mirror) and $\sinc(\bar\omega-u)$ (constant). For part~(a), the constant and cosine part
and the sine part of the model decouple on the symmetric window; minimising over $d$ and then over the cosine coefficient (a Schur
complement) gives $\ell^\ast$ in closed form, and $\dot u$ is a sum of two terms whose signs are those of two explicit functions
$N_0$ and $N_-$. Both vanish at $u=0$ and have the sign of $-u$ on half a step: for $\bar\omega\ge8\pi$ and $\bar\omega\ge6\pi$,
respectively, by a mean-value argument with far-field bounds, and below by rigorous interval arithmetic. Hence the field points
toward the target on half a step, the solution remains there, and it converges. For part~(b), the mirror and constant terms are
far-field slopes bounded by $\gamma$; where the main term dominates, the sign follows, and the barrier follows from the
one-dimensional comparison argument at $u_\pm$ (\proofref{real}).

Compared with Proposition~\ref{prop:isolated}, the loss contains the overlaps with the target's mirror line, with the line's own
mirror and with the constant term. In part~(b) the barrier therefore lies between $u_-$ and $u_+$, at a position that depends on
the phases and on the distance of the target from zero frequency; the condition on $\delta_{\mathrm{fix}}$ holds only if the target is
several steps from zero frequency and $|d|$ is small (since $\sinc'(u_+)\le0.096$ for $u_+\le5$, it requires $\bar\omega\gtrsim8$ even for
$d=0$ and $A\to0$, and $\bar\omega\gtrsim15$, about five steps, for $d=0$, $A=1$ and $|\kappa|=1$). In part~(a) the constant is at its best response and therefore also absorbs any mean of the target; part~(a) holds for a
target with a constant offset. Part~(a) excludes targets within one step of zero frequency: there the mirror terms are of the
same order as the main term, and for targets within half a step of zero the nearest grid line is the constant term, which no
prefactor moves. For targets one to eight steps from zero frequency, one step of the proof of part~(a) is computer-assisted. Far from zero
frequency, Proposition~\ref{prop:isolated} is the limit of Proposition~\ref{prop:real}: $\gamma(z)\to0$ as $z\to\infty$, so for bounded
$A$ and $d$ the bound $\delta_{\mathrm{fix}}$ vanishes, and for $\kappa\neq0$ the barrier of part~(b) can be placed arbitrarily close to
$u_1$, as in Proposition~\ref{prop:isolated}(b); numerically, the reach of part~(a) is within a few percent of one step already five
steps from zero (Remark~\ref{rem:realline}). On the windows of four and ten periods of Section~\ref{sec:single-line} the target lies
$21$ and $53$ steps from zero frequency, so there Proposition~\ref{prop:isolated} gives the relevant numbers.

\begin{remark}[Numerical comparison with the complex line]\label{rem:realline}
Evaluated from the closed-form loss, the first stationary point of the real line with best-response amplitude and constant lies
between $0.95$ and $1.10$ steps for targets one to three steps from zero frequency and between $0.98$ and $1.03$ steps at five
steps (Table~\ref{tab:realline}), close to the one-step reach of the complex line. The barrier of part~(b) departs further. The
mirror terms are overlaps of lines about $2\bar\omega$ apart and decay like $1/\bar\omega$ (Corollary~\ref{cor:bound}), but near zero
frequency they shift the barrier to between about $1.1$ and $2$ steps, depending on the phases (for $|\kappa|\ge0.2$); it remains
beyond half a step.
\end{remark}

\subsection{Grid initialisation}
\label{sec:theory-grid}

The following corollary gives prefactors that place a line within the reach of Propositions~\ref{prop:isolated}(a)
and~\ref{prop:real}(a) of every target in a band.

\begin{corollary}[Grid condition]\label{cor:grid}
Let $g>0$ and $k\ge1$.
\begin{enumerate}[label=(\alph*),leftmargin=1.6em,itemsep=2pt]
\item The ternary prefactors $(g,3g,\dots,3^{k-1}g)$ produce exactly the lines $mg$, $m\in\mathbb Z$, $2|m|\le3^k-1$, each from
  exactly one digit vector.
\item Every $\omega$ with $2|\omega|\le g(3^k-1)$ lies within $g/2$ of one of these lines. If $g\le\pi/T$, the gap to that line
  satisfies $|u|\le\pi/2$; if $g<2\pi/T$, it satisfies $|u|<\pi$.
\item With $g=\pi/T$, $k=\lceil\log_3(2N+1)\rceil$ gates suffice for every target within $N$ steps of zero frequency.
\end{enumerate}
\end{corollary}

\noindent\emph{Proof sketch.} Part~(a) for $g=1$ is the spectrum of the exponential encoding of~\citet{Shin_2023}; that each line arises from
exactly one digit vector is the balanced-ternary representation of the integers~\citep{knuth1997art}; parts~(b) and~(c)
follow by rounding $\omega/g$ to the nearest integer (\proofref{grid}).

Combined with Proposition~\ref{prop:real}(a), a grid with spacing one step starts each target at least one step from zero
frequency within the reach of an isolated real line. The corollary concerns initial positions only. On such a grid the target's second-nearest line also lies within one
step, so a line on the grid is not isolated, and its convergence in a model with several lines is not covered by the
propositions (Section~\ref{sec:coupling}). The condition is sufficient, not necessary: a target outside the reach of every line can
still be fitted by several lines together, and an in-phase line with fixed amplitude can converge from up to $u_1/\pi\approx1.43$
steps (Proposition~\ref{prop:isolated}(b)). On $[-\pi,\pi]$ the integer grid ($g=1$)
has a spacing of one step.

\begin{remark}[Choice of spacing]\label{rem:spacing}
For a complex line, the corollary permits spacings up to two steps, but doubling the spacing saves at most one gate: the number
of gates grows as $\log_3$ of the number of lines, so halving that number removes $\log_32\approx0.63$ gates, and a whole gate only
for bands just above a power of three. Wider spacings also move the farthest start toward the edge of the reach, where the
target-driven field vanishes. By Proposition~\ref{prop:isolated}(a), the speed of a complex line with best-response amplitude is
$|\dot u|=2T^2|\sinc(u)\sinc'(u)|$; it is maximal at $0.41$ steps, $96\%$ of the maximum at half a step (the farthest start on a grid
with spacing one step), and $47\%$, $15\%$ and $0\%$ of the maximum at $0.75$, $0.9$ and $1$ step.
\end{remark}

\section{From one line to the grid}
\label{sec:single}

This section tests the predictions of Section~\ref{sec:theory} in three steps. Section~\ref{sec:single-line} checks, for a single
line, that distances count in resolution steps (for cold runs only after a warm-up, see below) and that lines stop at the side lobes $u_k$ (Proposition~\ref{prop:isolated}(a)),
that inside the first side lobe a line starts toward its target exactly when it is in phase with it
(Propositions~\ref{prop:isolated}(b) and~\ref{prop:real}(b)), that a line with an amplitude close to its best response arrives from
within one step (Proposition~\ref{prop:isolated}(a); half a step is proved for a real line, Proposition~\ref{prop:real}(a)), and
that lines starting beyond the first side lobe do not come in (Proposition~\ref{prop:isolated}(b); for a real line, under the
condition of Proposition~\ref{prop:real}(b)). Section~\ref{sec:single-spacing} measures how far apart
grid lines may be, which Corollary~\ref{cor:grid}(b) bounds by two steps, and Section~\ref{sec:single-init} compares the grid start of
Section~\ref{sec:theory-grid} with random starts. The propositions assume gradient flow on the loss over the continuous window, one
model line and one target line, and an amplitude that is either best response or fixed. The experiments depart from all four
assumptions: they use Adam with minibatches on a finite training sample and train ansatz parameters and prefactors jointly; the
single-line experiments of Section~\ref{sec:single-line} run on one qubit, which has exactly one real line and a constant as in
Proposition~\ref{prop:real}, whereas the circuits of Sections~\ref{sec:single-spacing} and~\ref{sec:single-init} have many coupled
model lines and targets with three lines. The experiments therefore show how far the predictions carry over to these conditions.

\paragraph{Setup.}
Circuits alternate full SU$(2^n)$ blocks with one layer of encoding gates $R_X(\alpha_jx)$, one per qubit, and read out $Z$ on the
last qubit. Inputs are equidistant points on the window; on $[-\pi,\pi]$ one resolution step is one frequency unit. Targets are sums
of cosines with random phases, scaled to $[-1,1]$, each run with its own random $80/20$ train/test split. Prefactors and ansatz
parameters are trained with Adam~\citep{kingma2015adam} (learning rate $10^{-3}$, minibatches of $40$) in three arms: \emph{fixed},
prefactors frozen ($7000$ steps); \emph{cold}, joint training from the start ($7000$ steps); and \emph{warm}, $2000$ steps on the
ansatz parameters alone, which bring the amplitudes close to their best response, then $5000$ joint steps. Comparisons are paired
(same target, split and ansatz seed); where several seeds share a target, we test over target means. Details are in
Appendix~\ref{app:details}.

\subsection{Single-line experiments}
\label{sec:single-line}

\paragraph{Resolution step and side lobes.}
Section~\ref{sec:theory} predicts that distances count in resolution steps $\pi/T$, not in frequency units, and that a line which
does not reach its target stops at a side lobe $u_k$. A single-qubit circuit fits $\cos(5.3x)$ on windows $[-D\pi,D\pi]$ with $D\in\{1,4,10\}$, so that the same starts, $0.05$ to
$0.7$ frequency units below the target, lie $0.05$ to $7$ steps away. Lines that did not arrive stopped at the side lobes predicted
in steps, on every window: at $2.42$ and $2.46$ steps (median for $D=4$ and $D=10$; predicted $u_2/\pi=2.459$), $4.45$ ($4.477$) and
$6.48$ ($6.484$); all lines that started within $1.2$ steps arrived (Figure~\ref{fig:e0}b). In these runs, joint training with Adam
stopped where the gradient flow of Proposition~\ref{prop:isolated}(a) has its stationary points. (Starts at $1$, $2$, $3$, $5$ and $7$ steps lie on zeros of $\sinc$, where the
best-response pull vanishes, so where these lines go depends on how the amplitude changes.)

\paragraph{Starting direction and warm-up.}
Propositions~\ref{prop:isolated}(b) and~\ref{prop:real}(b) predict that, inside the first side lobe, a line starts toward its
target exactly when it is in phase with it. In a study with random target phases, starts on both sides of the target and windows of
one and four periods, this holds in $97.7\%$ of $3072$ runs across all arms. Cold starts are in phase about half the time, so at half
a step only $54\%$ start toward the target and $73\%$ arrive. A warm-up of $2000$ steps brings the amplitude close to its best
response, which is in phase within one step: then $98\%$ start toward the target and $99\%$ arrive (Figure~\ref{fig:warmup}).
Between one step and the first side lobe the best response is out of phase, and the warm-up turns lines away, as predicted; a grid
with spacing one step (Corollary~\ref{cor:grid}) places no target there. Without a warm-up, distance in steps does not fully determine arrival: from $0.3$ to $0.9$ steps, $77\%$ of cold runs arrive on
the window of one period and $59\%$ on the window of four (Fisher's exact test, $p=2\times10^{-5}$); after a warm-up the windows
agree (Table~\ref{tab:window} in Appendix~\ref{app:single}).

\begin{figure}[t]\centering
\includegraphics[width=\linewidth]{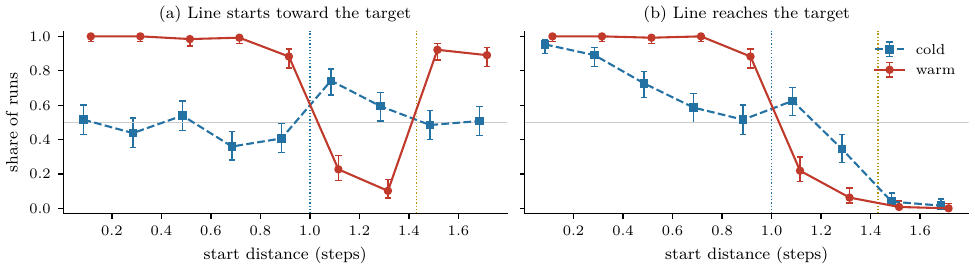}
\caption{The warm-up sets the starting direction of a single line. One target line, single-qubit circuit, line started at a given
distance on either side ($128$ runs per point and arm). (a) Share of runs whose line starts to move toward the target. (b) Share
whose line reaches it. Within one step, cold starts point either way and often fail; after a warm-up nearly every line starts
toward the target and arrives. Between one step and the first side lobe the warm-up turns lines away, as predicted. 95\% Wilson
intervals; dotted: one step and the first side lobe of a complex line.}
\label{fig:warmup}
\end{figure}

\paragraph{Beyond the first side lobe.}
Of $768$ runs that started $1.5$ or $1.7$ steps from the target, beyond the first side lobe, $11$ arrived. Eight of them do not
satisfy the barrier condition of Proposition~\ref{prop:real}(b) at the start of joint training, because their phase is close to
quadrature, where the target-driven pull vanishes, or because the target is only about five steps from zero frequency. In the other
three, restarts with one assumption restored at a time trace the crossing to the finite training sample and to Adam carrying the
line past the side lobe while the coefficients rotate into phase (\appref{app:barrier}{barrier crossings}).

\paragraph{Robustness.} The reach does not depend on the
learning rate (Adam from $3\times10^{-4}$ to $10^{-2}$, and plain SGD where it is stable; \appref{app:lr}{learning rate and optimiser}), and shot noise in parameter-shift
gradients of the prefactor, with the ansatz frozen after an exact warm-up, down to $100$ shots per evaluation, does not change arrival and changes the median final distance by less than
$0.001$ steps (\appref{app:shots}{shot noise}). The
condition that the target be at least one step from zero frequency is sufficient but not sharp: lines reach targets $0.75$ steps
from zero as often as farther ones, but not targets $0.25$ steps from zero (\appref{app:nearzero}{targets near zero frequency}).

\subsection{Grid spacing}
\label{sec:single-spacing}

Corollary~\ref{cor:grid} guarantees a start within one step, the reach of a complex line, for spacings below two steps, and
within the half-step reach proved for a real line for spacings up to one step. We measured the cost of wider spacings on
targets with three lines, where the model lines are no longer independent: each prefactor enters many model lines, so changing it
moves all of them, and Propositions~\ref{prop:isolated} and~\ref{prop:real} describe each line only as if the others were absent.
We scaled the ternary prefactors of a three-qubit circuit to $(g,3g,9g)$, so that the model lines sit at $mg$, $m=-13,\dots,13$, and
drew, for each $g\in\{0.5,1,1.5,2,2.5,3,4\}$, $32$ targets of three lines, each in a different cell and displaced from its centre
by a fraction of the spacing drawn uniformly from $(-\frac12,\frac12)$ (selection rules in \appref{app:experiments}{experiments}). Offsets and
phases are the same for every spacing.

\paragraph{Fit and arrival.} We report two measures that need not agree: the held-out $R^2$ of the whole fit, and
\emph{arrival}, the share of target lines that some model line reaches (final distance below $0.1$ step). Where two model lines
lie within one step of a target line, a good fit does not require either of them to arrive. With trained prefactors the median
held-out $R^2$ is at least $0.983$ for every spacing up to two steps and drops to $0.62$--$0.66$ at $g=3$ and $0.57$ at $g=4$; the
fixed grid falls from $0.990$ at $g=0.5$ to $0.94$ at $g=1$ and $0.53$ at $g=1.5$ (Figure~\ref{fig:e7}a).

On the wide grids ($g\ge1.5$) each target line has one model line nearby, and the fit depends on its arrival. Target lines that
started within $0.75$ steps were reached in $95\%$ (cold) and $99\%$ (warm) of cases, as for a single line, and at a given start
distance below $0.75$ steps the spacing has no detectable effect (\appref{app:arrspacing}{arrival and spacing}). Despite this, the
fit degrades at wide spacings, because target lines that start more than one step from every model line are mostly missed
($29$--$33\%$ reached from $1$ to $1.5$ steps, Figure~\ref{fig:e7}b): runs in which all three target lines start within one step
reach a median $R^2$ of at least $0.97$ at every spacing up to three steps, runs with at least one target line beyond one step $0.58$--$0.68$ (median per spacing and arm, $g=2.5$ and $3$).

On the dense grids ($g\le1$) the reverse holds: every start is within half a step, yet only $55$--$74\%$ of the target lines are
reached, while the fit is good. The others end $0.15$--$0.19$ steps away (median) with their components recovered to at least
$0.987$: the two nearby model lines share the target, the residual is small, and so is the pull on each line, which is proportional
to the residual (Eq.~\eqref{eq:dalpha}). With three times the budget these lines move closer but rarely arrive
(\appref{app:denselong}{longer training on dense grids}).

\begin{figure}[t]\centering
\includegraphics[width=\linewidth]{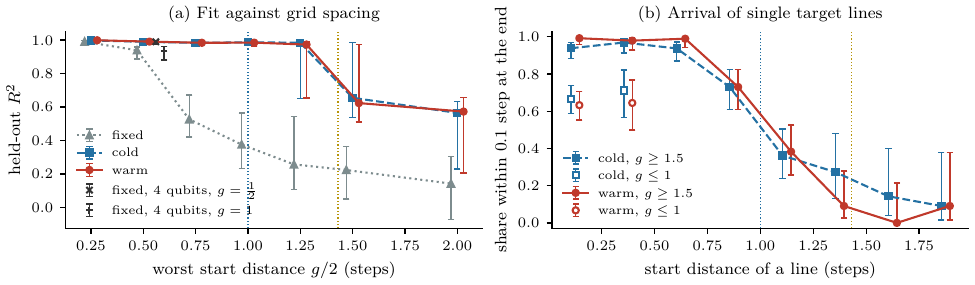}
\caption{Grid spacing $g$, with prefactors $(g,3g,9g)$; $32$ targets of three lines per spacing. (a) Held-out $R^2$ (median
and interquartile range) against the worst start distance $g/2$. Fixed: the same grid with frozen prefactors. Black: fixed
four-qubit grids on the targets for $g=1$, with spacing half a step ($\times$) and one step ($+$). (b) Individual target lines,
regrouped by their start distance: share that a model line reaches (final distance below $0.1$ step), on the grids with
$g\ge1.5$ (filled) and $g\le1$ (open; on these grids every start is within half a step). 95\% Wilson intervals; dotted: one step and the first side lobe of a complex line.}
\label{fig:e7}
\end{figure}

\paragraph{A finer fixed grid.} On the targets for $g=1$ we also fitted fixed grids with half the spacing. Covering the band
$1$--$13$ in half-steps takes a fourth ternary prefactor, since $\frac12(1,3,9)$ reaches only $6.5$; the grid
$\frac12(1,3,9,27)$ has $81$ lines instead of $27$; in the paper circuit,
with one encoding gate per qubit, this adds a qubit and raises the number of ansatz parameters from $126$ to $510$. As a control
with the same circuit and spacing one step we fitted the four-qubit integer grid $(1,3,9,27)$. The larger circuit alone does not
help: the integer grid fits as well as the three-qubit fixed grid (median paired difference $+0.001$, Wilcoxon $p=0.5$). The finer
spacing does: the half-step grid improves on the four-qubit integer grid by $0.058$ (all $32$ targets) and fits as well as the
trained three-qubit grid (median paired difference $-0.001$, $15$ of $32$ targets positive, $p\ge0.5$). On these targets, the three-qubit
grid with trained prefactors thus fits as well as a fixed grid with three times as many lines and four times as many ansatz
parameters.

\paragraph{Choice of spacing.} A spacing of one step keeps every start within half a step, where Proposition~\ref{prop:real}(a)
holds for an isolated line, at the cost of at most one gate more than two steps (Remark~\ref{rem:spacing}). The experiments do not
single it out: from one to two steps the trained fits are similar (success $78$--$94\%$; one step against $1.5$ or $2$, Fisher
$p\ge0.3$, on targets that differ between grids), and on the grid with spacing one step a good fit does not require a single line
to arrive.

\subsection{Grid versus random initialisation}
\label{sec:single-init}

We fitted $20$ targets of three lines in $1$--$13$ (five ansatz seeds each) from the ternary grid $(1,3,9)$ and from random starts
(\appref{app:init}{grid versus random initialisation}). The strongest of these, the \emph{band-scaled random start}, draws
random proportions scaled so that the model lines stay inside the band and reach its edge. Against the band-scaled start, the grid has the higher mean $R^2$ over the five seeds on $17$ of $20$ targets in both
arms (two-sided Wilcoxon signed-rank test, $W=16$ and $20$, $p=3\times10^{-4}$ and $7\times10^{-4}$). It reduces failures rather
than improving the typical fit: medians differ by at most $0.005$, while the success rate ($R^2\ge0.95$) rises from $69$ to $84\%$
and the 10th percentile from $0.56$ to $0.93$ (warm arm; Table~\ref{tab:e8}). The failures of random starts concentrate in runs where some target line starts more
than one step from every model line ($1$--$29\%$ success in warm runs, against $86$--$100\%$ otherwise); a classical model with $13$
independently trainable frequencies shows the same split (\appref{app:classical}{classical baselines}). Training on top of the grid raises the
success rate from $34\%$ (fixed grid) to $84\%$ (median paired gain $+0.049$). Setting the band edge too low loses the target lines
above it; setting it at $40$ or $121$ instead of $13$ costs at most $0.001$ in $R^2$ on this circuit (\appref{app:misspec}{misspecified band}).

\paragraph{The unit start in a setting following~\citet{Jaderberg_2024}.}
\citet{Jaderberg_2024} start with all prefactors equal to $1$. On the targets above, the unit start $(1,1,1)$ covers only the
lines $0$ to $3$ of a band that reaches $13$ and fails for that reason ($5\%$ success), which says nothing about reach. We therefore
tested it in a setting following their Figure~3, where it covers the band (four qubits, their
hardware-efficient ansatz, one to seven target frequencies in $[1,3]$ on $[-4\pi,4\pi]$; \appref{app:jaderberg}{setting following Jaderberg et al.}; not a
reproduction). There one step is $0.25$ frequency units, so the integer lines of the unit start are four steps apart. The prefactors change by a median of $2$--$5\%$ (largest change
per run) on the targets with four to six lines, and target lines that start off the integer lines are reached in $20\%$ of runs from $0.8$ steps, $12\%$ from $1.33$, $5\%$ from $1.6$ and none from $2$ steps. Since
only $20\%$ arrive even from $0.8$ steps, inside the reach of an isolated line, these data do not separate the distance of the start
from the coupling of up to seven target lines through four prefactors. From the ternary grid with spacing one step, $0.25\,(1,3,9,27)$,
the median MSE on the targets with four to seven lines is $0.005$--$0.013$, against $0.023$--$0.064$ from the unit start and
$0.009$--$0.014$ for the same grid with frozen prefactors.

\section{Several target lines: coupled model lines}
\label{sec:coupling}

A change of $\alpha_j$ moves all model lines with $t_j\neq0$ (Section~\ref{sec:setting}); these lines \emph{share} the
prefactor $\alpha_j$. The self term, in addition, couples every pair of model lines (Section~\ref{sec:theory-self}). This section
first measures how the fit depends on the number of target lines and what limits it (Section~\ref{sec:coupling-complexity}), and
then examines the placements that shared prefactors permit (Section~\ref{sec:coupling-levers}), the size of both couplings at the
start of training (Section~\ref{sec:coupling-mixing}), and the number of target lines placed by training
(Section~\ref{sec:coupling-nlines}).

\subsection{Fit against target complexity}
\label{sec:coupling-complexity}

We drew $m$ target lines at random, each at a distinct integer in the band plus an offset uniform in $[-\frac12,\frac12]$ steps, on
the circuits with prefactors $(1,3,9)$ ($k=3$, band $1$--$13$, $m\le13$) and $(1,3,9,27)$ ($k=4$, band $1$--$40$, $m\le24$), $16$
targets per $m$, with equal amplitudes; variants with random amplitudes and with noise are in
\appref{app:complexity}{fit against target complexity}. To separate the limit set by the available lines from the limit of reaching their coefficients, we also fitted the target in every
run by least squares on all lines of the starting grid and on all lines at the final prefactors. Every circuit with given prefactors outputs a function spanned by their
lines, so on the training points the least-squares $R^2$ is an upper bound for any such circuit; on held-out points it is a
reference.

\begin{figure}[t]\centering
\includegraphics[width=\linewidth]{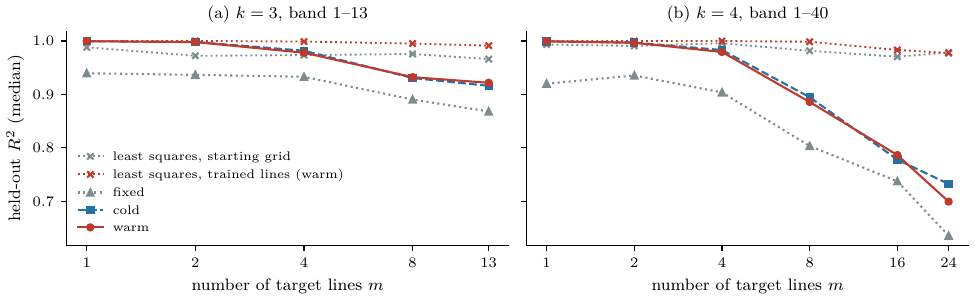}
\caption{Held-out $R^2$ (median over $16$ targets) against the number $m$ of target lines, equal amplitudes, for (a) $k=3$ and (b)
$k=4$ ternary prefactors started on the integer grid. Fixed, cold and warm: the circuit's three arms. Dotted: least squares on all
lines of the starting grid and on all lines at the final prefactors of the warm runs.}
\label{fig:complexity}
\end{figure}

The lines of the fixed grid suffice for nearly every target: least squares on them reaches a median held-out $R^2$ of
$0.97$--$1.00$ for every $m$, also with $24$ target lines, and $0.98$--$1.00$ on the training points (Figure~\ref{fig:complexity}).
The fixed circuit stays below this, from $0.94$ for one target line to $0.64$ for $24$, and comes within $0.01$ of it in only $5\%$
of runs. Training the prefactors improves the circuit's held-out $R^2$ by a median of $0.03$--$0.08$ (better in $93$--$96\%$ of
runs). At most $0.03$ of this comes from better lines: with few target lines, the moved lines raise the least-squares value to
$1.00$, all that is left above the grid; with many, they leave it unchanged. The rest is a smaller shortfall below what the lines
allow: with one or two target lines the trained circuits reach their least-squares value, with eight or more they
remain $0.05$--$0.22$ below it.

To test whether these shortfalls are a matter of optimisation, we reran the fixed and warm arms on eight targets each with $k=4$
and $m\in\{4,8,24\}$ with three and ten times the joint budget (\appref{app:budget}{training budget}). The shortfalls on the training
points shrink in both arms, by $40$--$80\%$ for a tenfold increase in steps (fixed, $m=24$: $0.20$, $0.15$, $0.12$), so they are
not a fixed limit of the ansatz within this range. Training the prefactors remains better at every budget (median paired gain
$0.04$--$0.09$, at least $7$ of $8$ targets), but with $8$ and $24$ target lines the fixed circuit trained ten times longer fits
better than the trained circuit at the original budget (median paired difference $0.01$ and $0.09$), with $4$ target lines it does not.
With many target lines, the fit is thus limited by reaching the coefficients, not by the lines available, at least up to ten times the budget. Training the prefactors
eases this, since one moved line replaces many coefficients, but it does not remove it; the following subsections show why few
lines are placed.

\subsection{Placement by the prefactors}
\label{sec:coupling-levers}

For the target line with frequency $\omega_i$, $i=1,\dots,m$, let $\mathbf t_i\in\{-1,0,1\}^E$ be the digit pattern of the nearest
model line, and $\delta_i=\omega_i-\mathbf t_i\cdot\boldsymbol\alpha$ the \emph{offset} between them. The digit patterns form the rows
of a matrix $P\in\{-1,0,1\}^{m\times E}$. A prefactor change $\Delta\boldsymbol\alpha$ moves these model lines by
$P\Delta\boldsymbol\alpha$, so all of them coincide with their target lines if and only if
\begin{equation}
  P\,\Delta\boldsymbol\alpha=\boldsymbol\delta .
  \label{eq:levers}
\end{equation}
\emph{Exact placement} is possible if Eq.~\eqref{eq:levers} has a solution. This holds for every offset vector if and
only if $\operatorname{rank}P=m$, which requires $m\le E$; for particular offsets it can also hold when $\operatorname{rank}P<m$,
if the offsets are consistent with the patterns (model lines at $1$, $3$ and $4=1+3$ can be placed exactly if
$\delta_3=\delta_1+\delta_2$). Otherwise the least-squares solution $\Delta\boldsymbol\alpha^\ast=P^+\boldsymbol\delta$ leaves line $i$
a residual distance $|(\boldsymbol\delta-P\Delta\boldsymbol\alpha^\ast)_i|$ from its target.

Equation~\eqref{eq:levers} thus predicts which placements the shared prefactors permit; the following experiments test whether
training reaches them. A \emph{geometry} fixes the grid lines near which the target lines lie, and with them $P$, and how their
offsets are drawn; we used one geometry for each of the following cases of sharing and solvability: one prefactor per line, a shared prefactor,
opposite-sign sharing, a rank-deficient $P$ with consistent and with conflicting offsets, and more target lines than the rank of
$P$. They are examples, not an exhaustive set. On the three-qubit ternary circuit $(1,3,9)$ with these six geometries of two to four target lines (Table~\ref{tab:levers},
\appref{app:geometries}{designed geometries}), model lines end $0.04$--$0.08$ steps from their targets (median per geometry) where exact placement is
possible, and $0.14$ and $0.17$ steps where it is not, of the size of the least-squares residual ($0.20$ and $0.13$ steps). When
exact placement is impossible, the target amplitudes decide which line stays farthest away: a line with one third of the others'
amplitude ends farthest ($0.32$ steps), a line with three times their amplitude closest ($0.07$), as least squares weighted by the
squared amplitudes predicts. In every geometry training reaches a median held-out $R^2$ of $0.990$--$0.999$, against
$0.920$--$0.975$ for the fixed grid. The fit stays high even where exact placement is impossible: a line $0.14$ steps from its
target still overlaps it strongly ($\sinc(0.14\pi)\approx0.97$), and the loss rewards the fit, not the positions of the lines.

\subsection{Decomposition of the starting velocity}
\label{sec:coupling-mixing}

Under gradient descent, line $i$ moves with velocity
$v_i=-|\mathbf t_i|^2F_{\mathbf t_i}-\sum_{\mathbf t\neq\mathbf t_i}(\mathbf t_i\cdot\mathbf t)F_{\mathbf t}$ (Eq.~\eqref{eq:mixing},
\appref{app:velocity}{velocity decomposition}), where the sum runs over lines with positive frequency and $F_{\mathbf t}$ is the rate at which the loss changes when line $\mathbf t$ alone
is moved, together with its mirror line: its \emph{own term}, the velocity it would have if it alone moved, which the single-line theory describes, plus a
\emph{mixing term} from all model lines that share a prefactor with it. Independently, the velocity splits into a target-driven and
a self part (Section~\ref{sec:theory-self}). We computed both splits exactly at the start of joint training for every run of the
grid-spacing experiment. Both couplings are large: mixing supplies a median $37\%$ (warm) and $68\%$ (cold) of the magnitude of the
velocity, and on the grid with spacing one step the self part is a median $0.62$--$0.69$ of the target-driven part (at most $0.38$ at
spacings of $1.5$ steps and more). They also change directions: after the warm-up, the own term points toward the target for
$84$--$100\%$ of lines within $0.75$ steps, as Proposition~\ref{prop:isolated} predicts, but mixing reverses the net velocity for
$16\%$ of them ($35\%$ cold), most often on the dense grids, and on the grid with spacing one step the self part reverses the
target-driven direction for $16\%$ of lines.

These numbers describe the velocity at one instant and do not show how much the couplings cost. In warm runs, lines whose own term
is reversed by mixing arrive less often than lines whose own term and net velocity both point toward the target ($74\%$ against
$89\%$), but reversals are most frequent on the dense grids, where arrival is lower for another reason
(Section~\ref{sec:single-spacing}); in cold runs, where the coefficients change within the first $100$ joint steps, the starting
direction does not predict arrival (\appref{app:warmup3}{warm-up with several target lines}). In the circuit neither coupling can be
switched off, so their cost cannot be measured directly. A classical model with independent frequencies, which has the self term but
no shared prefactors, succeeds from the grid in every run on the targets of Section~\ref{sec:single-init}, against $84\%$ for the
circuit (\appref{app:classical}{classical baselines}), but it also differs in how its amplitudes are set.

\subsection{Number of target lines}
\label{sec:coupling-nlines}

On the targets of Section~\ref{sec:coupling-complexity} with equal amplitudes, we counted the target lines that training places. A
target line counts as \emph{placed} if it started at least $0.1$ step from its nearest model line and ends within $0.1$ step; the
references below count with the same rule. The median number placed per run is one to two for $k=3$ at every $m$ and one to three
and a half for $k=4$ (Table~\ref{tab:nlines}). Placement within $0.1$ step is not bounded by $k$, since one prefactor change can
bring several target lines close; some prefactor change (by at most $2$ each) places up to $6.5$ ($k=3$, $m=13$) and $10.5$ ($k=4$,
$m=24$) target lines (median). Training places far fewer: up to $16$ target lines, on average about as many as the least-squares
change of Section~\ref{sec:coupling-levers} (mean difference $-0.15$ warm, $0.00$ cold). With $24$ target lines it places more
($3.5$ warm, $4$ cold, against $1$; medians), but not by a better placement. There the offsets largely cancel in the least-squares
change, which moves the model lines by a median of $0.09$ steps, whereas training moves them by $0.16$ (warm). A prefactor change of
the trained size in a random direction (components permuted, signs flipped) places nearly as many target lines through their
initially nearest model lines ($2.8$ against $3.2$ per run, warm); other model lines place another $0.3$
(\appref{app:placedchance}{placed lines and chance}). From $13$ target lines on with $k=3$ and $16$ with $k=4$, training places at
most $0.4$ more lines through their initially nearest model lines than such a random change (mean per run). The maximum is a geometric bound that the loss does not reward: placing more target
lines exactly can leave the others farther from their model lines, whereas the loss weighs all of them. With ten or more target
lines, the mean distance of the target lines to their nearest model line decreases by $0.04$ steps or less, yet training improves
the median held-out $R^2$ over the fixed grid by $0.03$--$0.08$ (better in $75$--$100\%$ of paired runs): with many target lines,
the gain from training the prefactors does not come from placing model lines on target lines.

\section{Real data}
\label{sec:realdata}

On real data, training the frequencies can only help if the data contain lines off the grid by a resolvable amount. We selected
series with a classical screen and evaluated the circuits on held-out points that neither the screen nor the training saw, under a
protocol written before the screen or any circuit was run and an amendment written before its confirmatory runs
(\appref{app:registration}{protocol}).

\subsection{Protocol}
\label{sec:realdata-protocol}

\paragraph{Data and screen.}
The candidates are the $49$ univariate, regularly sampled, non-duplicate series with at least $120$ observations in the Rdatasets
collection~\citep{arelbundock2026rdatasets}, pinned to one commit. Each series (at most its last $512$ observations) is mapped to
$[-\pi,\pi]$, so that one step is one cycle per window, and detrended; a fixed random fifth of its points is the \emph{holdout}, on
which every $R^2$ in this section is computed, so the task is interpolation, not forecasting; the remaining four fifths, the
\emph{development set}, are used for screening and training. On the development set, the screen fits least squares on the integer grid of lines $1,\dots,W$ ($W=13$ or $40$, the reach of three- and four-qubit ternary
encodings) and again with the three strongest lines free to move by up to one step; the median gain in test $R^2$ over five splits is
the \emph{screen gain}. A series is selected if the grid fit reaches $R^2\ge0.6$ and the screen gain is at least $0.01$ (with checks
against trend and output range; at most four series). Four pass: monthly gas production (\emph{gas}, screen gain $0.078$), weekly
US gasoline supply (\emph{gasoline}, $0.022$), yearly sunspot numbers (\emph{sunspot}, $0.019$) and monthly US electricity production
(\emph{elec}, $0.013$); for the median series, freeing three lines adds only $0.002$. Controls were fixed by rule beforehand:
Nottingham monthly temperatures, twenty whole years with the yearly line on the grid, and the two \emph{screen negatives}, the series
with the best grid fit among those with a screen gain of at most $0.003$ (\emph{BJsales} and \emph{gasprice}).

\paragraph{Circuits and analysis.}
All seven series are fitted with the four-qubit ternary circuit $(1,3,9,27)$, whose reach $40$ covers both screen bands and the
yearly lines of the monthly series ($39.7$ cycles per window for gas and elec); a band set too high cost little in
Section~\ref{sec:single-init}. The circuit has one encoding layer (\emph{minimal}; $510$ ansatz parameters) and is trained in the three
arms of Section~\ref{sec:single}. Each run trains on a resample of three quarters of the development set. The differences between arms
(warm $-$ fixed, cold $-$ fixed) on the same resample and ansatz seed are averaged over three ansatz seeds, giving one value per
resample. The primary hypothesis, fixed in advance as the arm the theory favours, is that the warm arm improves on the fixed grid; it
is tested per selected series by a one-sided sign test over resamples, with the Holm correction for the four selected series, which
controls the family-wise error rate~\citep{holm1979simple}. The cold arm and the controls are secondary comparisons. Because five
resamples, as specified, could not yield a significant test after correction, an amendment written before further runs added $20$
fresh resamples, on which all tests are run (\appref{app:pilot}{pilot and amendment}). All resamples share the same holdout, so the
tests ask whether a gain is robust to the training subset and to training randomness, not whether it holds for new data from the same
series.

\begin{table}[t]\centering\small
\caption{Real series, minimal circuit, $20$ confirmatory resamples, held-out $R^2$. Screen gain on the development set; least
squares (LS) on the holdout: median gain from freeing three lines over the integer grid. Circuit: median over $60$ runs per arm;
paired gain of the warm arm over the fixed grid: mean of the $20$ resample means with their range, and the number of resamples with a
positive mean. $p$: one-sided sign test over resamples, Holm-adjusted (primary family for the selected series, secondary family
otherwise). $^\ast$LS on the integer grid fails on the BJsales holdout ($R^2<0$).}
\label{tab:realdata}
\resizebox{\linewidth}{!}{%
\begin{tabular}{llcccccccc}
\toprule
series & role & screen gain & LS gain & fixed & cold & warm & warm $-$ fixed & resamples $>0$ & $p$ (Holm)\\
\midrule
gas       & selected        & $0.078$  & $+0.052$ & $0.860$ & $0.874$ & $0.879$ & $+0.022$ $[-0.015,0.061]$ & $18/20$ & $0.0008$\\
gasoline  & selected        & $0.022$  & $+0.006$ & $0.584$ & $0.587$ & $0.584$ & $+0.001$ $[-0.042,0.026]$ & $10/20$ & $0.59$\\
sunspot   & selected        & $0.019$  & $+0.016$ & $0.763$ & $0.816$ & $0.814$ & $+0.045$ $[-0.013,0.110]$ & $17/20$ & $0.003$\\
elec      & selected        & $0.013$  & $+0.008$ & $0.845$ & $0.860$ & $0.861$ & $+0.011$ $[-0.007,0.028]$ & $18/20$ & $0.0008$\\
gasprice  & screen negative & $0.002$  & $+0.003$ & $0.942$ & $0.950$ & $0.954$ & $+0.010$ $[-0.005,0.022]$ & $19/20$ & $0.0003$\\
BJsales   & screen negative & $0.000$  & $\phantom{+}$n.a.$^\ast$ & $0.893$ & $0.886$ & $0.904$ & $+0.009$ $[-0.059,0.039]$ & $14/20$ & $0.63$\\
Nottingham & control        & $-0.026$ & $-0.022$ & $0.912$ & $0.914$ & $0.913$ & $+0.000$ $[-0.013,0.033]$ & $9/20$ & $1$\\
\bottomrule
\end{tabular}}
\end{table}

\subsection{Results}
\label{sec:realdata-results}

On three of the four selected series, the warm arm improves on the fixed grid (Table~\ref{tab:realdata}; per resample in
Figure~\ref{fig:realdata} in Appendix~\ref{app:realdata}): gas by a mean of $0.022$ held-out $R^2$ ($18$ of $20$ resamples
positive, Holm-adjusted $p=0.0008$), sunspot by $0.045$ ($17$ of $20$, $p=0.003$) and elec by $0.011$ ($18$ of $20$, $p=0.0008$); on
gasoline the gain is $0.001$ ($p=0.59$). The pilot agrees in sign and size, and the cold arm gains as much as the warm arm
(secondary, Holm-adjusted $p\le0.02$ on the same three series). The screen, however, does not separate the series: the screen
negative gasprice also gains, by $0.010$ ($19$ of $20$, $p=0.0003$), while BJsales and the Nottingham control do not. In an
exploratory check after the confirmatory runs, not part of the protocol, we trained the fixed and warm arms three times as long
(\appref{app:rdbudget}{training budget}). The gain on gasprice is one of budget: the fixed arm at three times the budget fits better
than the warm arm at the original budget on all $20$ resamples. On sunspot, gas and elec the warm arm at the original budget stays
ahead of the fixed arm at three times the budget ($0.041$, $0.009$ and $0.009$; $16$, $15$ and $14$ of $20$ resamples; Holm-adjusted
$p=0.02$ only for sunspot), and at three times the budget for both arms it gains $0.012$--$0.058$ ($18$--$19$ of $20$). As for the
synthetic targets with many lines (Section~\ref{sec:coupling-nlines}), the gains do not come from placing a model line on the off-grid
line the screen identified: on gas, the circuit's line nearest to the freed line at $39.7$ moves away from it during training, while
on elec it moves toward it, and the least-squares gain on the holdout is largest on gas while the circuit gains most on sunspot.

The off-grid yearly line at $39.7$ that the screen found on gas and elec is an artefact of windows that do not contain whole years;
trimming the windows would place it on the grid, as for the Nottingham control, which shows no gain. We did not trim, because the
protocol fixed the preprocessing before the screen. The gains on these two series therefore do not show off-grid structure that
preprocessing could not remove.

\section{Discussion}
\label{sec:discussion}

\paragraph{What training the frequencies adds.}
On $[-\pi,\pi]$ the integer frequencies form a complete Fourier basis, so an off-grid component can be approximated arbitrarily
well by grid lines, but only with many of them, since its coefficients decay inversely with the distance from it. In our
experiments the lines of a fixed ternary grid already suffice for a near-perfect fit in the least-squares sense, even for targets with
$24$ off-grid lines (Section~\ref{sec:coupling-complexity}). What limits the circuit is reaching the many coefficients this fit needs.
Trainable prefactors add economy: one line moved onto a component replaces many coefficients, and at equal budget the trained
circuits come closer to what their lines allow, while the moved lines themselves improve the least-squares fit by at most $0.03$. The
resulting gain over the fixed grid with one-step spacing is modest, $0.03$ to $0.09$ held-out $R^2$ on synthetic targets and $0.010$
to $0.045$ on the real series that gained. It can be matched by other means: on three-line targets by a fixed grid with half the
spacing (Section~\ref{sec:single-spacing}), and with $8$ and $24$ target lines by training the fixed grid ten times longer.

\paragraph{Practical guidance.}
Within the settings we tested (one-dimensional inputs, circuits of up to five qubits, synthetic targets and a few real series), our
results suggest the following.
(i)~\emph{Scale and spacing.} Map the input to $[-\pi,\pi]$, or use the spacing $\pi/T$ on a raw window $[-T,T]$, and start ternary
prefactors on the grid with spacing one resolution step, so that every target lies within half a step of a line
(Corollary~\ref{cor:grid}); on a raw window the integer grid leaves targets more than one step from every line once $T>2\pi$.
(ii)~\emph{Band.} Choose the band to cover the highest frequency expected in the target. With one encoding layer, a band three to nine
times too large cost at most $0.002$ in $R^2$ (Section~\ref{sec:single-init}); encodings whose number of lines far exceeded the number
of ansatz parameters and whose reach exceeded the sampling limit failed, for reasons our experiments do not separate
(\appref{app:encoding}{choice of encoding}).
(iii)~\emph{Warm-up.} Train the ansatz parameters alone first. For an isolated line this sets the starting direction and raised arrival
from half a step from $73\%$ to $99\%$ (Section~\ref{sec:single-line}); with several target lines, warm and cold training fitted
equally well, so the warm-up costs little.
(iv)~\emph{Baselines.} Compare a gain from training the prefactors with a fixed grid trained longer and with a fixed grid of half the
spacing, both of which matched it in some of our settings (Sections~\ref{sec:single-spacing} and~\ref{sec:coupling-complexity}).

\paragraph{Is a quantum circuit needed?}
Nothing in our results requires a quantum computer, and nothing shows an advantage of the circuit. The theory uses the circuit only
through its Fourier form, and a classical model with $13$ independently trainable frequencies shows the same split between starts
within and beyond one step (\appref{app:classical}{classical baselines}). On our synthetic targets that classical model succeeds from the grid in
every run with $40$ parameters, where the circuit succeeds in $84\%$ with $129$, and random Fourier features with a suitable
bandwidth fit almost perfectly, consistent with work on when Fourier-type quantum models can be approximated
classically~\citep{landman2022classically,schreiber2023classical,sweke2025potential}. Coupling through shared prefactors, absent
from the classical model, is a constraint rather than an advantage: it limits how many lines can be placed
(Section~\ref{sec:coupling}). As \citet{yu2022power} noted for expressivity, trainable encoding weights blur the line between classical
and quantum models; our results say how to train trainable-frequency circuits, not that they should be preferred.

\paragraph{Limitations.}
The theory covers an isolated line with best-response or fixed amplitude under gradient flow; the experiments use Adam and joint
training, and the effect of the self term and of shared prefactors on convergence is measured, not proved. For targets one to eight
steps from zero frequency, one step of the proof of Proposition~\ref{prop:real}(a) is computer-assisted. Inputs are one-dimensional,
circuits are small and simulated exactly, and shot noise is studied only for the final approach of a single line. The encoding
comparison does not separate reach from the number of degrees of freedom and from aliasing. On real data, the selection depended on the
sample, the tests are conditional on one holdout per series, the mechanism behind the gains is not shown, and the check of the training budget is exploratory and stops at three times the budget,
where on synthetic targets ten times was needed for the fixed grid to catch up. Our results therefore describe
one-dimensional, small and exactly simulated models; whether they carry over to multivariate inputs, larger circuits and hardware is
open.

\paragraph{Conclusion.}
Trainable-frequency quantum models promise to capture a component at a non-integer frequency with one moved frequency instead of
many coefficients. We have shown when training delivers this and what it adds. The target-driven part of the prefactor gradient is a
weighted sum of slopes of the overlap between model and target lines. This gives each line a reach: under gradient flow, an isolated
complex line with best-response amplitude converges from within one resolution step, and a real line with mirror line and constant from
within half a step, proved with a computer-assisted step and formalised in Lean~4; with a fixed amplitude and under a condition on the phase, a
barrier near the first side lobe stops lines that start farther out. A ternary grid of prefactors starts every target within half a step of a line with a number of
encoding gates logarithmic in the band.

In circuit experiments, single lines follow these predictions, and the start largely decides whether training the frequencies works:
from the grid, fits fail less often than from random starts ($84\%$ against $69\%$ success on three-line targets), whereas from the unit
start of prior work, targets more than a step from a line are rarely reached. With many target lines, the lines of the fixed grid
would already suffice, and the limit is reaching their coefficients; training the prefactors narrows this shortfall but places few
lines, since they share prefactors. On real series it improved held-out $R^2$ at equal budget by a few hundredths on four of seven series, without the
gains being traceable to the off-grid lines a classical screen identified. A classical model with independently trainable frequencies
behaves in the same way, so our results describe how to train such models rather than an advantage of the circuit.

Three questions remain open: how much the coupling through shared prefactors and through the self term costs, which the circuit does
not allow us to switch off; why encodings whose reach far exceeds the band fail, through the number of lines relative to the ansatz
or through aliasing; and how the picture changes for multivariate inputs, where the number of mixed frequencies grows quickly with
the number of features~\citep{Schuld_2021}, so that replacing many coefficients by a few moved lines could matter more. Code, results and the Lean~4 formalisation
(Appendix~\ref{app:lean}) are available at \coderef.

\bibliographystyle{tmlr}
\bibliography{references}
\appendix
\sloppy

\section{Lean formalisation}
\label{app:lean}

All results of Section~\ref{sec:theory} are stated and proved in Lean~4 with Mathlib (one file,
\texttt{ResolutionTheory.lean}, Mathlib \texttt{v4.28.0}), together with the prefactor derivative,
the loss split and the placement and mixing identities of Sections~\ref{sec:setting} and~\ref{sec:coupling}. The project builds
without \texttt{sorry}, and \texttt{Axioms.lean} reports for each of its $122$ main theorems, and hence for the $49$ auxiliary
lemmas they use, only the standard axioms \texttt{propext}, \texttt{Classical.choice} and \texttt{Quot.sound}. Two inputs are checked outside Lean, both sign
conditions of Proposition~\ref{prop:real}(a) that the reach theorem takes as hypotheses and that we verified by interval arithmetic
(Appendix~\ref{app:proofs}): the condition for the part of the loss without the constant (\texttt{SignCond}) for targets one to six
steps from zero frequency, and the condition for the part with the constant (\texttt{SignCondEven}) for targets one to eight steps
from zero. Beyond these ranges both are proved (\texttt{signCond\_of\_ge}, \texttt{signCondEven\_of\_ge}), and
\texttt{real\_br\_reach\_const\_far\_nu} needs no hypothesis. The statements are ours; the proofs were produced with AI assistance,
Aristotle (Harmonic) for the isolated complex line and the grid and Claude (Anthropic) for the real line, and are checked by
Lean. Convergence statements assume a solution of the gradient-flow equation and show where it goes; no existence or uniqueness
theorem is used. Proposition~\ref{prop:real}(a) is stated for the line frequency itself (\texttt{real\_br\_prop\_const\_nu}): the line stays within half
a step of the target (\texttt{real\_br\_trap\_const\_nu}), and hence away from zero frequency, which the change of variables to the gap
requires, and converges to it (\texttt{real\_br\_reach\_const\_nu}).

\begin{table}[h]\centering\small
\begin{tabular}{>{\raggedright\arraybackslash}p{0.38\linewidth}>{\raggedright\arraybackslash}p{0.55\linewidth}}
\toprule
Result & Lean theorems (namespace \texttt{Resolution})\\
\midrule
Equation~\eqref{eq:dalpha} (prefactor derivative) & \texttt{prefactor\_deriv}\\
Equation~\eqref{eq:loss} (target-driven and self terms) & \texttt{loss\_split}\\
$\sinc$, $\sinc'$; $\sinc'$ is the derivative & \texttt{sinc}, \texttt{dsinc}, \texttt{hasDerivAt\_sinc}, \texttt{sinc\_eq\_integral}, \texttt{dsinc\_eq\_integral}\\
Lemma~\ref{lem:kernel} (overlap slope) & \texttt{kernel\_eq}, \texttt{norm\_kernel}\\
Proposition~\ref{prop:exact} (exact target term) & \texttt{targetTerm\_eq}\\
Self term (Section~\ref{sec:theory-self}) & \texttt{targetTerm\_eq}, with the model in place of the target\\
Corollary~\ref{cor:bound} (two-sided bound) & \texttt{abs\_dsinc\_le\_near}, \texttt{abs\_mul\_dsinc\_le\_far}, \texttt{kernel\_bound}, \texttt{targetTerm\_single\_bound}\\
Proposition~\ref{prop:isolated}, loss and best response & \texttt{overlap\_eq}, \texttt{lineLoss\_eq\_integral}, \texttt{lineLoss\_best\_response}, \texttt{brLoss\_le\_one}, \texttt{brLoss\_kpi}, \texttt{gain\_eq}\\
Proposition~\ref{prop:isolated}(a), flow & \texttt{br\_flow\_in\_u}, \texttt{br\_basin}, \texttt{br\_stall}, \texttt{ripple\_unique}, \texttt{ripple\_mem}, \texttt{dsinc\_eq\_zero\_iff}, \texttt{cos\_ne\_zero\_of\_dsinc\_eq\_zero}\\
Proposition~\ref{prop:isolated}(b), fixed amplitude & \texttt{fixed\_flow\_in\_u}, \texttt{first\_stall\_exists}, \texttt{lineLoss\_fixed\_amp\_sign} ($0<u<u_1$), \texttt{lineLoss\_fixed\_amp\_sign\_neg} ($-u_1<u<0$), \texttt{fixed\_in\_phase}, \texttt{fixed\_out\_of\_phase}\\
Proposition~\ref{prop:isolated}(b), barrier & \texttt{fixed\_amp\_barrier\_pos}, \texttt{fixed\_amp\_barrier}, \texttt{fixed\_amp\_no\_reach}, \texttt{fixed\_amp\_barrier\_concrete}\\
Proposition~\ref{prop:real} (real line), loss and fields & \texttt{realLoss\_eq\_integral}, \texttt{realLoss\_hasDerivAt}, \texttt{real\_fixed\_flow\_in\_u}, \texttt{realLoss\_dpq}, \texttt{complete\_sq\_const}, \texttt{realBrLossConst\_is\_min}, \texttt{realBrLossConst\_zero}, \texttt{realBrLossConst\_hasDerivAt}, \texttt{abs\_sinc\_lt\_one}, \texttt{abs\_dsinc\_le\_half}, \texttt{abs\_dsinc\_le\_far'}\\
Proposition~\ref{prop:real}(a) & \texttt{half\_step\_pos}, \texttt{resC\_pos}, \texttt{varC\_pos}, \texttt{Nconst\_zero}, \texttt{Nminus\_zero}, \texttt{signCondEven\_of\_parts}, \texttt{signCond\_of\_parts}, \texttt{dNconst\_neg\_of\_ge}, \texttt{dN\_neg\_of\_ge}, \texttt{signCondEven\_of\_ge}, \texttt{signCond\_of\_ge}, \texttt{brFieldConst\_sign}, \texttt{real\_br\_reach\_const}, \texttt{real\_br\_reach\_const\_nu}, \texttt{real\_br\_trap\_const\_nu}, \texttt{real\_br\_prop\_const\_nu}, \texttt{real\_br\_reach\_const\_far\_nu}; the hypotheses of \texttt{signCondEven\_of\_parts} for $\pi\le\bar\omega\le8\pi$ and of \texttt{signCond\_of\_parts} for $\pi\le\bar\omega\le6\pi$ are checked by interval arithmetic\\
Proposition~\ref{prop:real}(b) & \texttt{fixedFieldReal\_sub\_le}, \texttt{real\_fixed\_direction}, \texttt{real\_fixed\_in\_phase}, \texttt{real\_fixed\_barrier}, \texttt{flow\_trap}\\
Corollary~\ref{cor:grid} (grid condition) & \texttt{grid\_basin\_general}, \texttt{balanced\_ternary\_surj}, \texttt{balanced\_ternary\_inj}, \texttt{balanced\_ternary\_bound}, \texttt{grid\_cover}, \texttt{grid\_in\_basin}, \texttt{gates\_suffice\_real}, \texttt{gates\_suffice}\\
Placement by shared prefactors (Section~\ref{sec:coupling}) & \texttt{placement\_iff}, \texttt{placement\_all\_iff\_rank}, \texttt{placement\_rank\_le}\\
Equation~\eqref{eq:mixing} (own and mixing terms) & \texttt{lineFreqs}, \texttt{linePull}, \texttt{prefactor\_chain\_rule}, \texttt{velocity\_split}, \texttt{mixing\_identity}\\
One-dimensional flows (used throughout) & \texttt{flow\_tendsto\_left}, \texttt{flow\_tendsto\_right}\\
\bottomrule
\end{tabular}
\end{table}

\clearpage
\section{Proofs}
\label{app:proofs}

This appendix proves the results of Section~\ref{sec:theory}. Section~\ref{app:proofs-tools} collects properties of $\sinc$ and two
lemmas on one-dimensional flows; Section~\ref{app:proofs-main} proves the lemma, propositions and corollaries in the order of
Section~\ref{sec:theory}. Table~\ref{tab:notation} lists the notation.

\begin{table}[h]\centering\small
\caption{Notation used in Section~\ref{sec:theory} and in this appendix.}
\label{tab:notation}
\begin{tabular}{ll}
\toprule
Symbol & Meaning\\
\midrule
$T$ & half-width of the input window $[-T,T]$; one resolution step is $\pi/T$\\
$\nu$, $\omega$ & model line, target line (frequencies)\\
$u=(\omega-\nu)T$ & gap; $u=\pi$ is one step\\
$\bar\omega=\omega T$ & target distance from zero frequency, in units of $1/T$ ($\bar\omega/\pi$ steps)\\
$\sinc(z)=\sin z/z$ & overlap of two lines at gap $z$; $\sinc'$, $\sinc''$ its derivatives\\
$u_k$ & zero of $\sinc'$ in $(k\pi,(k+1)\pi)$, $k\ge1$ (side lobe); $u_1\in(4.49,4.50)$\\
$c$, $r=\operatorname{Re}c$ & complex amplitude of a complex line, its in-phase part\\
$d$, $A$, $\varphi$ & constant, amplitude and phase of a real line; $\psi$ phase of the target\\
$\kappa=\cos(\varphi-\psi)$ & phase factor of a real line with fixed amplitude\\
$a,b,\mu$ & $\sinc(u)$, $\sinc(2\bar\omega-u)$, $\sinc(2\bar\omega-2u)$: overlaps with target, target mirror, own mirror\\
$s$, $s_0$ & $\sinc(\bar\omega-u)$, $\sinc(\bar\omega)$: overlaps of the constant with the model line and the target\\
$R$, $D$ & $\tfrac{a+b}2-s_0s$ and $\tfrac{1+\mu}2-s^2$ (residual overlap and variance, Lemma~\ref{lem:brconst})\\
$\gamma(z)$ & $(z+1)/z^2$, far-field bound on $|\sinc'(z)|$\\
$g$ & grid spacing\\
\bottomrule
\end{tabular}
\end{table}

\subsection{Auxiliary results}
\label{app:proofs-tools}

\begin{lemma}[Properties of $\sinc$]\label{lem:sinc}
For all $z\in\mathbb R$:
\begin{enumerate}[label=(\roman*),leftmargin=1.6em,itemsep=1pt]
\item $\sinc(z)=\int_0^1\cos(sz)\,ds$, $\sinc'(z)=-\int_0^1s\sin(sz)\,ds$ and $\sinc''(z)=-\int_0^1s^2\cos(sz)\,ds$;
\item $|\sinc(z)|\le1$, $|\sinc'(z)|\le\min\{|z|/3,\,1/2\}$, $|\sinc''(z)|\le1/3$ and $|z\,\sinc'(z)|\le2$;
\item for $z>0$: $|\sinc(z)|\le1/z$, $|\sinc'(z)|\le\gamma(z)$ and $|\sinc''(z)|\le(z^2+2z+2)/z^3$, and the three bounds decrease in $z$;
\item for $|z|\le\pi/2$: $\sinc(z)\ge2/\pi$ and $\sinc''(z)\le-m_2$ with $m_2=((\pi/2)^2-2)/(\pi/2)^3>0.12$.
\end{enumerate}
\end{lemma}

\begin{proof}
(i) $\int_0^1\cos(sz)\,ds=\sin z/z$ for $z\neq0$ and $1$ for $z=0$; differentiation under the integral gives the other two
identities. (ii) follows from (i) with $|\cos|\le1$, $|\sin(sz)|\le\min\{s|z|,1\}$ and $\int_0^1s^2ds=\frac13$; moreover
$z\,\sinc'(z)=\cos z-\sinc(z)$, so $|z\,\sinc'(z)|\le2$. (iii) For $z>0$, $\sinc(z)=\sin z/z$, $z\,\sinc'(z)=\cos z-\sinc(z)$ and
$\sinc''(z)=-\sinc(z)-2\sinc'(z)/z$; the bounds follow, and each is a sum of negative powers of $z$ with positive coefficients.
(iv) Jordan's inequality gives $\sin|z|\ge2|z|/\pi$. For $s\in[0,1]$ and $|z|\le\pi/2$, $\cos(sz)\ge\cos(s\pi/2)$, so by (i)
$-\sinc''(z)\ge-\sinc''(\pi/2)=2/\pi-16/\pi^3=m_2$, and $m_2>0.12$ for $\pi>3.1415$.
\end{proof}

\noindent The constant $2$ in (ii) is not sharp; numerically, $\sup_z|\cos z-\sinc(z)|\approx1.063$.

\begin{lemma}[Sign structure of $\sinc'$]\label{lem:signs}
$\sinc'$ is odd and negative on $(0,\pi]$. For every integer $k\ge1$, $\sinc'$ has exactly one zero $u_k$ in $(k\pi,(k+1)\pi)$; it
satisfies $k\pi<u_k<k\pi+\pi/2$ and $\tan u_k=u_k$, and $(-1)^k\sinc'$ is positive on $(k\pi,u_k)$ and negative on $(u_k,(k+1)\pi)$.
In particular $u_1\in(4.49,4.50)$, $\sinc'<0$ on $(0,u_1)$ and $\sinc'>0$ on $(u_1,2\pi)$.
\end{lemma}

\begin{proof}
Let $\phi(z)=z\cos z-\sin z$, so $\sinc'(z)=\phi(z)/z^2$ for $z\neq0$ and $\phi'(z)=-z\sin z$. On $[k\pi,(k+1)\pi]$ the sign of $\sin$
is $(-1)^k$, so $(-1)^k\phi$ is strictly decreasing there, from $k\pi$ to $-(k+1)\pi$; it therefore has exactly one zero $u_k$ in
$(k\pi,(k+1)\pi)$, is positive before and negative after it. Since $(-1)^k\phi(k\pi+\pi/2)=-1<0$, $u_k<k\pi+\pi/2$; since
$\phi(u_k)=0$ and $\cos u_k\neq0$, $\tan u_k=u_k$. Since $\sinc'=\phi/z^2$ has the sign of $\phi$, $(-1)^k\sinc'$ is positive
before $u_k$ and negative after it. For $k=0$, $\phi$ decreases from $0$ on $[0,\pi]$, so $\sinc'<0$ on $(0,\pi]$. Oddness of $\sinc'$ follows from
evenness of $\sinc$. The enclosure $\phi(4.49)<0<\phi(4.50)$ follows from Taylor bounds for $\sin$ and $\cos$ around $3\pi/2$ and
$3.141592<\pi<3.141593$; since $\phi$ increases on $[\pi,2\pi]$, it locates $u_1$.
\end{proof}

\begin{lemma}[Barrier]\label{lem:barrier}
Let $F:\mathbb R\to\mathbb R$, $c\in\mathbb R$, and let $u:[0,\infty)\to\mathbb R$ be differentiable. Suppose that $\dot u(\tau)=F(u(\tau))$
for every $\tau\ge0$ with $u(\tau)=c$. If $u(0)\le c$ and $F(c)<0$, then $u(\tau)\le c$ for all $\tau\ge0$. If $u(0)\ge c$ and
$F(c)>0$, then $u(\tau)\ge c$ for all $\tau\ge0$.
\end{lemma}

\begin{proof}
Suppose $u(0)\le c$, $F(c)<0$ and $u(\tau_1)>c$ for some $\tau_1>0$. Let $\tau_0=\sup\{\tau\in[0,\tau_1]:u(\tau)\le c\}$. By continuity
$u(\tau_0)=c$, $\tau_0<\tau_1$ and $u>c$ on $(\tau_0,\tau_1]$. But $\dot u(\tau_0)=F(c)<0$, so $u(\tau)<c$ for $\tau>\tau_0$ close to
$\tau_0$, a contradiction. The second statement follows by applying the first to $-u$ and $\tilde F(v)=-F(-v)$.
\end{proof}

\noindent The lemma requires the differential equation only at times when $u$ is on the barrier; this is used in the proof of
Proposition~\ref{prop:real}(a).

\begin{lemma}[Convergence]\label{lem:convergence}
Let $p<q$ and $\varepsilon>0$, let $F$ be continuous on $[p,q]$ with $F<0$ on $(p,q]$, $F(p)=0$ and $F>0$ on $(p-\varepsilon,p)$, and let
$u$ satisfy $\dot u=F(u)$ for all $\tau\ge0$.
\begin{enumerate}[label=(\roman*),leftmargin=1.6em,itemsep=1pt]
\item If $u(0)\in(p,q]$, then $u$ is non-increasing and $u(\tau)\to p$ as $\tau\to\infty$.
\item If moreover $F<0$ on $(p,p+\varepsilon)$ and $u(0)=p$, then $u(\tau)=p$ for all $\tau\ge0$.
\end{enumerate}
The statements with all inequalities reversed hold for solutions that start below a zero of $F$.
\end{lemma}

\begin{proof}
(i) By Lemma~\ref{lem:barrier} at $u(0)$, where $F(u(0))<0$, $u\le u(0)$; by Lemma~\ref{lem:barrier} at
$p-\eta$ for every $\eta\in(0,\varepsilon)$, $u\ge p-\eta$, hence $u\ge p$. On $[p,u(0)]$, $F\le0$, so $u$ is non-increasing and
bounded, and converges to some $L\ge p$. If $L>p$, then $F\le-m<0$ on $[L,u(0)]$ by continuity and compactness, so
$u(\tau)\le u(0)-m\tau$, which contradicts $u\ge L$. (ii) Lemma~\ref{lem:barrier} at $p\pm\eta$ for every $\eta\in(0,\varepsilon)$.
\end{proof}

\subsection{Proofs of the results of Section~\ref{sec:theory}}
\label{app:proofs-main}

\begin{proof}[Proof of Lemma~\ref{lem:kernel}]\phantomsection\label{prf:kernel}
For $n=0$ both sides vanish. For $n\neq0$, integration by parts gives
$\int_{-T}^{T}x e^{inx}dx=2i\bigl(\sin(nT)/n^2-T\cos(nT)/n\bigr)$. Dividing by $2T$ and writing $z=nT$ gives
$i(\sin z-z\cos z)/(Tn^2)=-iT\,(z\cos z-\sin z)/z^2=-iT\sinc'(nT)$.
\end{proof}

\begin{proof}[Proof of Proposition~\ref{prop:exact}]\phantomsection\label{prf:exact}
Insert $\partial_{\alpha_j}f=ixG_j$ (Eq.~\eqref{eq:dalpha}) and the expansions of $h$ and $G_j$ into
$\mathcal T_{\mathrm{tgt}}=-\frac1{2T}\int_{-T}^{T}h\,\partial_{\alpha_j}f\,dx$. The pair of a target line $\omega$ and a model line $\nu_{\mathbf t}$
contributes $-i\,h_\omega t_jc_{\mathbf t}\,\frac1{2T}\int_{-T}^{T}x\,e^{i(\omega+\nu_{\mathbf t})x}dx=-T\,h_\omega t_jc_{\mathbf t}\sinc'((\omega+\nu_{\mathbf t})T)$
by Lemma~\ref{lem:kernel}.
\end{proof}

\begin{proof}[Proof of Corollary~\ref{cor:bound}]\phantomsection\label{prf:bound}
By Lemma~\ref{lem:sinc}(ii), $T|\sinc'(nT)|\le T\cdot T|n|/3$ and $T|\sinc'(nT)|=|nT\sinc'(nT)|/|n|\le2/|n|$. The bound on
$|\mathcal T_{\mathrm{tgt}}^{(\omega)}|$ follows from Proposition~\ref{prop:exact} and the triangle inequality, with $K_T(n)\le2/|n|\le2/\delta$.
\end{proof}

\begin{proof}[Proof of Corollary~\ref{cor:self}]\phantomsection\label{prf:self}
$\mathcal T_{\mathrm{self}}=\frac1{2T}\int f\,\partial_{\alpha_j}f\,dx$ is $-\mathcal T_{\mathrm{tgt}}$ with $h$ replaced by $f$; apply
Proposition~\ref{prop:exact} with the lines $\nu_{\mathbf t'}$ and coefficients $c_{\mathbf t'}$ of $f$ in place of the target, and
Corollary~\ref{cor:bound} to each term.
\end{proof}

\begin{proof}[Proof of Proposition~\ref{prop:isolated}]\phantomsection\label{prf:isolated}
\emph{Loss.} $|ce^{i\nu x}-e^{i\omega x}|^2=|c|^2+1-2r\cos((\omega-\nu)x)-2\operatorname{Im}c\,\sin((\omega-\nu)x)$. The sine integrates to zero
over $[-T,T]$ and $\frac1{2T}\int_{-T}^{T}\cos((\omega-\nu)x)\,dx=\sinc(u)$, so $\ell=|c|^2-2r\sinc(u)+1=(r-\sinc u)^2+(\operatorname{Im}c)^2+1-\sinc(u)^2$.
The minimum over $c$ is attained exactly at $c=\sinc(u)$, with value $1-\sinc(u)^2$.

\emph{(a)} By the chain rule, $\partial_\nu\ell^\ast=2T\sinc(u)\sinc'(u)$, so $\dot u=-T\dot\nu=2T^2\sinc(u)\sinc'(u)=:F(u)$; $F$ is odd and
continuous. On $(0,\pi)$, $\sinc>0$ and $\sinc'<0$ (Lemma~\ref{lem:signs}), so $F<0$ there and $F>0$ on $(-\pi,0)$; Lemma~\ref{lem:convergence}
with $p=0$ gives $u\to0$ for $0<u(0)<\pi$, and oddness gives the case $-\pi<u(0)<0$. On $(k\pi,(k+1)\pi)$, $\sinc$ has the sign of
$(-1)^k$ and, by Lemma~\ref{lem:signs}, $\sinc'$ has the sign of $(-1)^k$ before $u_k$ and of $(-1)^{k+1}$ after it; hence $F>0$ on
$(k\pi,u_k)$ and $F<0$ on $(u_k,(k+1)\pi)$. Lemma~\ref{lem:convergence} at $p=u_k$ from either side, and part~(ii) of that lemma for
$u(0)=u_k$, give $u\to u_k$; negative starts follow by oddness.

\emph{(b)} With $c$ fixed, $\partial_u\ell=-2r\sinc'(u)$ and $\dot u=-T^2\partial_u\ell=2rT^2\sinc'(u)$. Statement (i) is the sign of $\sinc'$
on $(0,u_1)$ (Lemma~\ref{lem:signs}) and its oddness. For (ii), if $r>0$ the field is negative on $(0,u_1)$ and positive on $(-u_1,0)$,
and Lemma~\ref{lem:convergence} gives $u\to0$; if $r<0$ the field is positive on $(0,u_1)$ and negative on $(u_1,2\pi)$, and the same
lemma at $p=u_1$ (from below) gives $|u|\to u_1$. For (iii), let $u(0)>u_1$ (the case $u(0)<-u_1$ is symmetric). If $r>0$, the field is
positive on $(u_1,2\pi)$, and Lemma~\ref{lem:barrier} at every $c'\in(u_1,\min\{u(0),2\pi\})$ gives $u\ge c'>u_1$. If $r<0$, it is positive on
$(0,u_1)$, and Lemma~\ref{lem:barrier} at $u_1-\eta$ for every $\eta>0$ gives $u\ge u_1$. If $r=0$, $u$ is constant.
\end{proof}

\begin{lemma}[Closed-form real-line loss]\label{lem:realloss}
With $\bar\omega=\omega T$ and $u=(\omega-\nu)T$, so that $\nu T=\bar\omega-u$, $(\omega+\nu)T=2\bar\omega-u$ and $2\nu T=2\bar\omega-2u$,
\[
\begin{aligned}
\ell(\nu;d,A,\varphi)={}&d^2+2dA\cos\varphi\,\sinc(\bar\omega-u)-2d\cos\psi\,\sinc(\bar\omega)+\tfrac{A^2}2\bigl(1+\cos2\varphi\,\sinc(2\bar\omega-2u)\bigr)\\
&-A\bigl(\cos(\varphi-\psi)\sinc(u)+\cos(\varphi+\psi)\sinc(2\bar\omega-u)\bigr)+\tfrac12\bigl(1+\cos2\psi\,\sinc(2\bar\omega)\bigr).
\end{aligned}
\]
\end{lemma}

\begin{proof}
Expand the square with $\cos\alpha\cos\beta=\frac12(\cos(\alpha-\beta)+\cos(\alpha+\beta))$ and $\cos^2\alpha=\frac12(1+\cos2\alpha)$, and use
$\frac1{2T}\int_{-T}^{T}\cos(mx+\beta)\,dx=\cos\beta\,\sinc(mT)$.
\end{proof}

\begin{lemma}[Best response with constant]\label{lem:brconst}
Let $\bar\omega\ge\pi$ and $|u|\le\pi/2$, and write $a=\sinc(u)$, $b=\sinc(2\bar\omega-u)$, $\mu=\sinc(2\bar\omega-2u)$, $s=\sinc(\bar\omega-u)$,
$s_0=\sinc(\bar\omega)$, $R=\frac{a+b}2-s_0s$ and $D=\frac{1+\mu}2-s^2$. Then $R>0$, $D>0$, $a-b>0$ and $|\mu|<1$, and
\[
  \ell^\ast(\nu)=\tfrac12\bigl(1+\cos2\psi\,\sinc(2\bar\omega)\bigr)-\cos^2\psi\Bigl(s_0^2+\frac{R^2}{D}\Bigr)-\sin^2\psi\,\frac{(a-b)^2}{2(1-\mu)},
\]
where the minimum over $d$, $A\ge0$ and $\varphi$ is attained. Moreover, $\ell^\ast$ is differentiable at $\nu$, and
$-T^2\partial_u\ell^\ast=T^2F(u)$ with
\[
  F(u)=\cos^2\psi\,\frac{R\,N_0(u)}{D^2}+\sin^2\psi\,\frac{(a-b)\,N_-(u)}{2(1-\mu)^2},
\]
\[
  N_0=2R'D-RD',\qquad N_-=2\bigl(\sinc'(u)+\sinc'(2\bar\omega-u)\bigr)(1-\mu)-2(a-b)\,\sinc'(2\bar\omega-2u),
\]
where primes denote derivatives with respect to $u$, and $N_0(0)=N_-(0)=0$.
\end{lemma}

\begin{proof}
With $P=A\cos\varphi$ and $Q=-A\sin\varphi$ the model line is $d+P\cos\nu x+Q\sin\nu x$, and $(P,Q)$ ranges over $\mathbb R^2$. On $[-T,T]$ the
constant and the cosines are even and the sines odd, so by Lemma~\ref{lem:realloss} $\ell$ is the sum of a quadratic form in $(d,P)$
and one in $Q$, minimised separately. The odd part is $\frac{Q^2}2(1-\mu)+Q\sin\psi\,(a-b)+\text{const}$, with minimum
$-\sin^2\psi\,(a-b)^2/(2(1-\mu))$ when $1-\mu>0$. The even part is $d^2+2dPs-2d\cos\psi\,s_0+\frac{P^2}2(1+\mu)-P\cos\psi\,(a+b)+\text{const}$;
completing the square in $d$ and then in $P$ gives
\[
  (d+Ps-\cos\psi\,s_0)^2+D\Bigl(P-\cos\psi\,\frac RD\Bigr)^2-\cos^2\psi\Bigl(s_0^2+\frac{R^2}D\Bigr)+\text{const},
\]
with minimum $-\cos^2\psi\,(s_0^2+R^2/D)$ when $D>0$. Here $D=\frac1{2T}\int\cos^2\nu x\,dx-(\frac1{2T}\int\cos\nu x\,dx)^2$ is the variance of
$\cos\nu x$ over the window.

\emph{Signs.} By Lemma~\ref{lem:sinc}(iv) and (iii), $a\ge2/\pi$, $|b|\le1/(2\bar\omega-u)\le2/(3\pi)$, $|s_0|\le1/\pi$, and, since
$z:=\bar\omega-u\ge\pi/2$, $|s|\le2/\pi$; also $2\bar\omega-2u\ge\pi$, so $|\mu|<1$. Hence $a-b>0$ and
$R\ge\frac12(a-|b|)-|s_0||s|\ge\frac2{3\pi}-\frac2{\pi^2}=\frac2\pi\bigl(\frac13-\frac1\pi\bigr)>0$. For $D$,
$2z^2D=z^2+z\sin z\cos z-2\sin^2z$. If $z\ge1.69$, this is at least $z^2-z/2-2>0$. If $\pi/2\le z<1.69$, then $\cos z\le0$,
$0\le\sin z\le1$ and $\cos z\ge-(z-\pi/2)$, so $z\sin z\cos z\ge z\cos z\ge-z(z-\pi/2)$ and $2z^2D\ge z\pi/2-2\ge\pi^2/4-2>0$.

\emph{Derivative.} $a,b,\mu,s$ are differentiable in $u$, and $D,1-\mu\neq0$, so $\ell^\ast$ is differentiable; with
$\partial_u(R^2/D)=R(2R'D-RD')/D^2$ and $\partial_u\frac{(a-b)^2}{1-\mu}=(a-b)N_-/(1-\mu)^2$ (using $b'=-\sinc'(2\bar\omega-u)$ and
$\mu'=-2\sinc'(2\bar\omega-2u)$) the formula for $F$ follows. At $u=0$, $a=1$, $b=\mu=\sinc(2\bar\omega)$ and $s=s_0$, so $R=D$, and
$R'(0)=D'(0)/2$ by differentiation; hence $N_0(0)=0$. Likewise $N_-(0)=0$ because $\sinc'(0)=0$ and $a-b=1-\mu$ at $u=0$.
\end{proof}

\begin{lemma}[Sign conditions]\label{lem:signcond}
For every $\bar\omega\ge\pi$ and $0<|u|\le\pi/2$, $N_0(u)$ and $N_-(u)$ have the sign of $-u$.
\end{lemma}

\begin{proof}
Since $N_0(0)=N_-(0)=0$, it suffices to show, for each function, either that its derivative in $u$ is negative on $[-\pi/2,\pi/2]$, or
that it is negative on $[-r,r]$ and that the function has the stated sign on $r\le|u|\le\pi/2$ (mean value theorem).

\emph{$N_-$ for $\bar\omega\ge6\pi$.} The main term of $N_-'$ is $2\sinc''(u)(1-\mu)\le-2m_2(1-\mu)$ (Lemma~\ref{lem:sinc}(iv)). The other
terms involve $\sinc$, $\sinc'$ and $\sinc''$ at arguments of at least $2\bar\omega-\pi$, bounded by Lemma~\ref{lem:sinc}(iii), together
with $|\sinc'(u)|\le\frac12$ and $|\sinc(u)|\le1$. For $\bar\omega=6\pi$ their sum is smaller than the main term by $0.016$, and all bounds
decrease in $\bar\omega$.

\emph{$N_0$ for $\bar\omega\ge8\pi$.} Write $N_0'=2R''D+R'D'-RD''$. Its main term is $\sinc''(u)D\le-m_2D$, and
$D\ge\frac12-\frac12|\mu|-s^2\ge0.487$. The rest involves $\sinc$, $\sinc'$ and $\sinc''$ at arguments of at least $\bar\omega-\pi/2$ (the
overlap with the constant), $2\bar\omega-\pi/2$ and $2\bar\omega-\pi$, bounded by Lemma~\ref{lem:sinc}(iii), and $|s_0|\le1/\bar\omega$. At
$\bar\omega=8\pi$ these terms sum to at most $0.047$, against a main term of at least $0.058$, and all bounds decrease in $\bar\omega$. The
overlap with the constant decays like $1/(\bar\omega-\pi/2)$ rather than $1/(2\bar\omega)$, which is why this threshold exceeds the one for
$N_-$ by two steps.

\emph{Computer-assisted part.} For $\pi\le\bar\omega\le6\pi$ (for $N_-$) and $\pi\le\bar\omega\le8\pi$ (for $N_0$) the remaining claims are
verified with rigorous interval arithmetic (Arb via python-flint, 80-bit balls; $\sinc$ and its derivatives near $u=0$ by Taylor series
with a bounded remainder): $N_-'<0$ for $|u|\le\pi/5$ and $N_-$ has the stated sign for $\pi/5\le|u|\le\pi/2$, on a cover of about
$19\,000$ boxes; $N_0'<0$ for $|u|\le\pi/8$ and $N_0$ has the stated sign for $\pi/8\le|u|\le\pi/2$, on a cover of about $194\,000$
boxes. Each box is certified by evaluating the function on the whole box in interval arithmetic and checking that the resulting
enclosure lies strictly on one side of zero. The scripts are in the code archive (\coderef).
\end{proof}

\begin{proof}[Proof of Proposition~\ref{prop:real}]\phantomsection\label{prf:real}
\emph{(a)} Let $F$ be as in Lemma~\ref{lem:brconst}. By Lemmas~\ref{lem:brconst} and~\ref{lem:signcond}, $F$ is continuous on
$[-\pi/2,\pi/2]$, negative on $(0,\pi/2]$, positive on $[-\pi/2,0)$, and $F(0)=0$. Wherever $|u(\tau)|\le\pi/2$, $\nu T\ge\bar\omega-\pi/2\ge\pi/2>0$,
so $\ell^\ast$ is differentiable at $\nu(\tau)$ and $\dot u=-T\dot\nu=T^2F(u)$. Lemma~\ref{lem:barrier}, which needs the equation only at
times when $u=\pm\pi/2$, therefore gives $|u(\tau)|\le\pi/2$ for all $\tau\ge0$; in particular $\nu$ never reaches zero frequency, and
$\dot u=T^2F(u)$ holds for all $\tau\ge0$. Lemma~\ref{lem:convergence} (with $p=0$, and part~(ii) for $u(0)=0$) gives $u\to0$.

\emph{(b)} By Lemma~\ref{lem:realloss}, $\dot u=-T^2\partial_u\ell=T^2A(\kappa\,\sinc'(u)+e(u))$ with
$e(u)=-\cos(\varphi+\psi)\sinc'(2\bar\omega-u)+A\cos2\varphi\,\sinc'(2\bar\omega-2u)+2d\cos\varphi\,\sinc'(\bar\omega-u)$. For $|u|\le5$ the three
arguments are at least $2\bar\omega-5$, $2\bar\omega-10$ and $\bar\omega-5$, all positive, and Lemma~\ref{lem:sinc}(iii) gives
$|e(u)|\le\delta_{\mathrm{fix}}$. Statement (i) follows, using Lemma~\ref{lem:signs} for the sign of $\sinc'$ on $(0,u_1)$. For (ii),
$\delta_{\mathrm{fix}}>0$, so the hypothesis implies $\kappa\neq0$; if $A=0$ the line does not move. Let $u(0)\ge u_+$. If $\kappa>0$, the field
at $u_+$ is at least $T^2A(\kappa\,\sinc'(u_+)-\delta_{\mathrm{fix}})>0$, and Lemma~\ref{lem:barrier} gives $u\ge u_+$. If $\kappa<0$, the field at
$u_-$ is at least $T^2A(|\kappa||\sinc'(u_-)|-\delta_{\mathrm{fix}})>0$, and Lemma~\ref{lem:barrier} gives $u\ge u_-$. For $u(0)\le-u_+$ the same
argument at $-u_+$ and $-u_-$ applies, since $\sinc'$ is odd. For instance, $\sinc'(4.3)<-0.04$, $\sinc'(4.7)>0.04$ and
$\delta_{\mathrm{fix}}<0.02$ for $\bar\omega\ge20\pi$, $A\le1$ and $d=0$, so a barrier then lies between $u=4.3$ and $u=4.7$ for $|\kappa|\ge1/2$.
\end{proof}

\begin{table}[h]\centering\small
\caption{Real line (Proposition~\ref{prop:real}), computed exactly from the closed-form loss: first stationary point beyond the
target with best-response amplitude and constant (the reach; complex line: $1$ step), over $18$ target phases, and position of
the barrier with fixed amplitude and $d=0$ (complex line: $1.430$ steps), over target and model phases ($18\times72$) and
amplitudes $A\in\{0.25,0.5,1\}$; in steps. For the target $1.3$ steps from zero only lines on the far side of the target from zero
frequency are included. The values are numerical; the barrier row for the target $1.3$ steps from zero
($\bar\omega\approx4.1$) lies outside Proposition~\ref{prop:real}(b), which assumes $\bar\omega>5$.}
\label{tab:realline}
\begin{tabular}{lccc}
\toprule
target, steps from zero & reach & barrier, $|\cos(\varphi-\psi)|\ge0.2$ & barrier, $|\cos(\varphi-\psi)|\ge0.5$\\
\midrule
$1.3$ & $0.95$--$1.08$ & $1.06$--$1.81$ & $1.14$--$1.73$\\
$2.3$ & $0.96$--$1.10$ & $1.11$--$1.78$ & $1.22$--$1.66$\\
$3.3$ & $0.97$--$1.06$ & $1.16$--$2.03$ & $1.23$--$1.93$\\
$5.3$ & $0.98$--$1.03$ & $1.23$--$1.91$ & $1.29$--$1.55$\\
$10.3$ & $0.99$--$1.02$ & $1.31$--$1.53$ & $1.36$--$1.49$\\
$21.2$ & $0.99$--$1.01$ & $1.38$--$1.48$ & $1.40$--$1.46$\\
$53.3$ & $1.00$ & $1.41$--$1.45$ & $1.42$--$1.44$\\
\bottomrule
\end{tabular}
\end{table}

\begin{proof}[Proof of Corollary~\ref{cor:grid}]\phantomsection\label{prf:grid}
(a) Every integer $m$ with $2|m|\le3^k-1$ is $\sum_{j<k}t_j3^j$ with $t_j\in\{-1,0,1\}$: write $m=3m'+t_0$ with $t_0\in\{-1,0,1\}$ and use
induction, since $2|m'|\le3^{k-1}-1$. Conversely $|\sum_jt_j3^j|\le\sum_j3^j=(3^k-1)/2$, and the representation is unique because two
representations would differ by a nonzero digit vector with value zero, whose lowest nonzero digit is not divisible by $3$. (b) For
$2|\omega|\le g(3^k-1)$, the nearest integer $m$ to $\omega/g$ satisfies $2|m|\le3^k-1$, since $3^k$ is odd, and $|\omega-mg|\le g/2$; hence
$|u|\le gT/2$, which is at most $\pi/2$ for $g\le\pi/T$ and less than $\pi$ for $g<2\pi/T$. (c) If $|\omega|\le Ng$ and $2N+1\le3^k$, then
$2|\omega|\le2Ng\le g(3^k-1)$.
\end{proof}

\section{Experimental details}
\label{app:details}

\paragraph{Circuits.} \emph{Paper circuit} on $n$ qubits: $|0\rangle$, a special-unitary block $e^{i\sum_m\theta_mP_m}$ over all
$4^n-1$ Pauli words, one encoding gate $R_X(\alpha_jx)$ per qubit, a second special-unitary block, readout $\langle Z\rangle$ on
the last qubit; $2(4^n-1)$ ansatz parameters ($n=1$: $6$; $n=3$: $126$; $n=4$: $510$). \emph{Re-uploading circuit} ($n=2$, $L$
encoding layers): $L+1$ special-unitary blocks interleaved with $L$ layers of two encoding gates, $15(L+1)$ ansatz parameters,
$E=2L$ encoding gates. Both are simulated as exact statevectors in JAX~\citep{frostig2018jax} in double precision; outputs agree with
PennyLane~\citep{bergholm2022pennylaneautomaticdifferentiationhybrid} to $3\times10^{-15}$ and gradients with finite differences to
$10^{-10}$.

\paragraph{Data and targets.} Single-line experiments: $100$ points per period $2\pi$ on $[-D\pi,D\pi]$; targets
$\cos(\omega x+\phi)$. All other synthetic experiments: $256$ points on $[-\pi,\pi]$; targets $\sum_iA_i\cos(\omega_ix+\phi_i)$
with unit amplitudes (unless stated) and random phases, scaled to $[-1,1]$. Each run has its own random $80/20$ split. The theory
uses the overlap on the continuous window; on the full set of sampled inputs, the empirical overlap of two lines agrees with
$\sinc(u)$ to within $0.011$ for every gap, with its first zero at $0.99$--$1.00$ steps, and restricting it to a training split adds
fluctuations of a few hundredths, so the sampled inputs do not change the resolution step.

\paragraph{Training.} Adam~\citep{kingma2015adam} (optax~\citep{optax2020github}), learning rate $10^{-3}$, minibatch $40$. Every arm
has $7000$ steps in total. Cold: $7000$ joint steps. Warm: $2000$ steps on $\theta$ only, then $5000$ joint steps (single-line
warm-up study: warm-ups of $500$ and $2000$ steps, then $6500$ and $5000$ joint steps; a warm-up of $10\,000$ steps is used only to
measure starting directions). Fixed: $7000$ steps on $\theta$ only. Ansatz parameters are initialised uniformly in $[0,2\pi)$. Prefactor initialisations
are stated per experiment.

\paragraph{Metrics.} Held-out $R^2$ on the test points; success means $R^2\ge0.95$. For each target line, its \emph{distance} is the
distance to the nearest accessible line $\mathbf t\cdot\boldsymbol\alpha$ (all digit patterns enumerated), in steps; a line
\emph{arrived} if its final distance is below $0.1$ step (single-line experiments: below $0.05$). \emph{Starts toward} is the
sign of the line's velocity $-\mathbf t\cdot\nabla_{\boldsymbol\alpha}\mathcal L$ at the start of joint training (full-batch
gradient). Off-target power is the fraction of the model output's variance on a fine grid not explained by the target lines.

\paragraph{Statistics.} Comparisons are paired (same target, split and ansatz seed). Where several seeds share a target we report
the number of targets whose mean favours a configuration and a two-sided Wilcoxon test~\citep{wilcoxon1945} over target means;
on real data, the number of splits with a positive mean and a one-sided sign test over splits.

\paragraph{Experiments.}\phantomsection\label{app:experiments} Single line (Section~\ref{sec:single-line}): one qubit, $\cos(5.3x)$, $D\in\{1,4,10\}$, starts $0.05$,
$0.1$, $0.2$, $0.3$, $0.5$ and $0.7$ frequency units below the target, $8$ seeds. Warm-up study: $\omega\in\{5,5.3\}$, random phase,
start $0.1$--$1.7$ steps on either side, $D\in\{1,4\}$, $16$ seeds, warm-ups of $0$, $500$, $2000$ and $10\,000$ steps. Grid
spacing: three qubits, prefactors $(g,3g,9g)$, $g\in\{0.5,1,1.5,2,2.5,3,4\}$, $32$ targets of three lines per spacing, one per grid
cell (Section~\ref{sec:single-spacing}); the cells are chosen so that the digit patterns of the three nearest model lines are
linearly independent (exact placement is possible for every offset, Section~\ref{sec:coupling-levers}), no target line is nearest
to the constant term, every target line is at least one step from zero frequency, and the target lines are at least $1.2$ steps
apart; fixed baselines with four qubits. Initialisation: three qubits, $20$ targets $\times$ $5$
seeds per family. Fit against target complexity: three and four qubits, $16$ targets per $m$ and variant; training budget: four
qubits, $8$ targets per $m$, $3$ and $10$ times the joint budget. Designed geometries and number of lines: three and four qubits, $16$ targets per condition. Encodings: re-uploading
circuit, $E\in\{4,6,8,14\}$, $16$ targets per family. Real data: four qubits, $5$ pilot and $20$ confirmatory resamples $\times$ $3$
seeds, scored on a fixed holdout. All experiments ran on the CPU of a laptop (Apple MacBook Pro, M2, $16$\,GB memory).

\paragraph{Code.} The code and notebooks for every experiment, table and figure, their results, and the Lean project of
Appendix~\ref{app:lean} are available at \coderef.

\section{Single-line experiments}
\label{app:single}

\begin{table}[h]\centering\small
\caption{Grid versus random initialisation ($k=3$; $100$ runs per cell: $20$ targets $\times$ $5$ seeds; three off-grid lines
per target). Median and 10th percentile of held-out $R^2$, and success rate ($R^2\ge0.95$).}
\label{tab:e8}
\begin{tabular}{llccc}
\toprule
start & arm & median & 10th pct. & success\\
\midrule
grid & fixed & 0.933 & 0.819 & 34\%\\
grid & cold  & 0.989 & 0.934 & 82\%\\
grid & warm  & 0.989 & 0.932 & 84\%\\
band-scaled random & cold & 0.989 & 0.589 & 70\%\\
band-scaled random & warm & 0.988 & 0.562 & 69\%\\
$U[1,9]$ & warm & 0.982 & 0.597 & 60\%\\
$U[0.5,4.5]$ & warm & 0.577 & $0.007$ & 23\%\\
$(1,1,1)$ & warm & 0.232 & $-0.092$ & 5\%\\
\bottomrule
\end{tabular}
\end{table}

\begin{table}[h]\centering\small
\caption{Arrival by window (single line, warm-up study): share of runs that reach the target, by start distance, on the windows
of one ($D=1$) and four ($D=4$) periods ($64$ runs per cell at $0.1$ steps, $128$ otherwise). Cold arrival depends on the
window at equal distance in steps; after a warm-up it does not.}
\label{tab:window}
\begin{tabular}{lcccccc}
\toprule
& \multicolumn{2}{c}{cold} & \multicolumn{2}{c}{warm-up $500$} & \multicolumn{2}{c}{warm-up $2000$}\\
\cmidrule(lr){2-3}\cmidrule(lr){4-5}\cmidrule(lr){6-7}
start (steps) & $D=1$ & $D=4$ & $D=1$ & $D=4$ & $D=1$ & $D=4$\\
\midrule
$0.1$ & $1.00$ & $0.91$ & $1.00$ & $0.98$ & $1.00$ & $1.00$\\
$0.3$--$0.5$ & $0.90$ & $0.72$ & $0.98$ & $0.97$ & $0.99$ & $1.00$\\
$0.7$--$0.9$ & $0.64$ & $0.46$ & $0.84$ & $0.85$ & $0.92$ & $0.96$\\
$1.1$--$1.3$ & $0.41$ & $0.56$ & $0.19$ & $0.20$ & $0.12$ & $0.16$\\
$1.5$--$1.7$ & $0.02$ & $0.03$ & $0.00$ & $0.02$ & $0.00$ & $0.01$\\
\bottomrule
\end{tabular}
\end{table}

\paragraph{Starting direction.} In the warm-up study, the agreement of $97.7\%$ (Section~\ref{sec:single-line}) counts the runs that
start inside the first side lobe, in all arms, including warm-ups of $10\,000$ steps ($3072$ runs). Longer warm-ups bring the amplitude closer to its best fit: at half a step the
mean distance to the best-fit coefficients (relative to their size) is $0.83$, $0.37$, $0.14$ and $0.04$ after $0$, $500$, $2000$ and $10\,000$ warm-up steps.

\paragraph{Setting following \citet{Jaderberg_2024}.}\phantomsection\label{app:jaderberg} From their Section~IV and Appendix~B: data
$y(x)=|\Omega|^{-1}\sum_{\omega\in\Omega}\cos(\omega x)$ at $\lceil2|D|\max\Omega\rceil$ equally spaced points of $D=[-4\pi,4\pi]$,
$\Omega$ with $1$ to $7$ frequencies equally spaced in $[1,3]$; four qubits; feature map $\exp(-ix\sum_m\alpha_mY_m/2)$; eight
layers of their ansatz of Figure~6a ($R_X R_Y R_X$ on every qubit, CNOT ring); cost operator $\sum_mZ_m$; Adam, learning rate
$10^{-3}$, batch $2$, $4000$ iterations, MSE on the data points. Not specified there, and chosen by us: the feature map acts on
$|0000\rangle$ before the ansatz, and the ansatz angles start uniform in $[0,2\pi]$. Ten seeds per data set. The fixed unit model
reaches a median MSE of $0.031$--$0.066$ on the data sets with off-grid frequencies, more than the range shown in their Figure~3,
so details of our circuit or data differ from theirs. Per target line, share reached from the unit start (final distance to the
nearest model line below $0.1$ step): from $0$ steps $88\%$ ($160$ lines), $0.8$ steps $20\%$ ($20$), $1.33$ steps $12\%$ ($60$),
$1.6$ steps $5\%$ ($20$), $2$ steps $0\%$ ($20$); median final distance $0.63$, $0.78$, $1.42$ and $1.77$ steps from the four
off-grid distances. Off-grid target lines occur only in the data sets with four to seven lines (at $0.8$ steps only with six), so
these data separate neither the start distance from the number of target lines nor either of them from the ansatz and batch size
of this setting.

\paragraph{Grid versus random initialisation.}\phantomsection\label{app:init} The four other starts are $(1,1,1)$, a common default; $U[1,9]$ per prefactor;
$U[0.5,4.5]$, which keeps every line inside the band; and random proportions (flat Dirichlet) scaled so that the prefactors sum to $13$, which
keeps the lines inside the band and reaches its edge. Targets are displaced up to half a step from distinct integers in
$1$--$13$. The grid wins against every random start (Table~\ref{tab:e8}, Figure~\ref{fig:e8}). A random start can be worse than
not training the frequencies: with $U[1,9]$, training has a higher median than the fixed grid ($0.982$ against $0.933$) but, because of its
failures, a lower mean on $15$ of $20$ targets. Training on top of the grid gives a median paired gain over the fixed grid of $+0.042$ (cold) and $+0.049$ (warm)
(higher on all $20$ targets in both arms). At equal budget, the warm and cold arms do not differ (paired difference $+0.0006$, $12$ of
$20$ targets, $p=0.8$), and the grid's advantage over the fairest random start holds in both arms. The
classical model with $13$ independently trainable frequencies shows the same split by start distance ($26$--$31\%$ against
$98\%$), succeeding in every run from the grid and in $70$--$72\%$ from random starts.

\begin{figure}[h]\centering
\includegraphics[width=\linewidth]{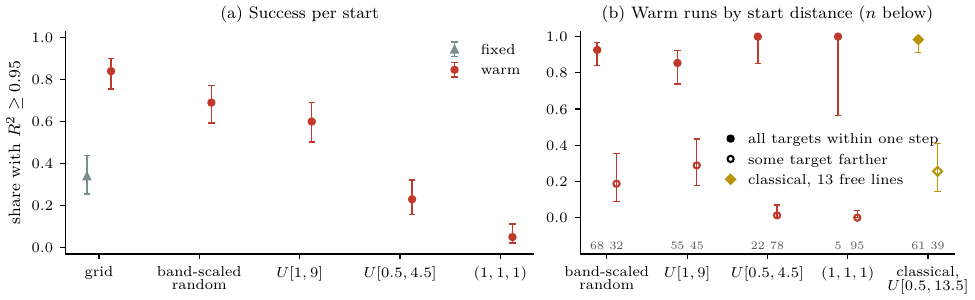}
\caption{Grid versus random initialisation. (a) Success rate ($R^2\ge0.95$) per start, trained from that start (warm), and for
the grid also with its prefactors frozen (fixed); the cold arm is within a few points of the warm arm (Table~\ref{tab:e8}).
(b) Warm runs from each random start, split by whether every target line starts within one step of a line (filled) or not
(open), with the number of runs below each point; right, the classical model with $13$ independently trainable frequencies from
random starts. Only $5$ of the $100$ runs from $(1,1,1)$ start every target within one step; those succeed. 95\% Wilson
intervals.}
\label{fig:e8}
\end{figure}

\paragraph{Barrier crossings.}\phantomsection\label{app:barrier} Of the $768$ warm-up-study runs that started $1.5$ or $1.7$ steps from the target, $11$ arrived;
they come from $8$ of the $256$ starting configurations, all with $\omega=5.3$ ($3$ runs on $D=1$, $8$ runs on $D=4$).
For each we fitted the line's coefficients ($A$, $\varphi$, $d$, relative to the target) at the start of joint training and
evaluated the condition of Proposition~\ref{prop:real}(b) with $u_-=4.3$ and $u_+=4.7$, which bracket the first side lobe
$u_1=4.49$: $|\cos(\varphi-\psi)|\min(|\sinc'(u_-)|,\sinc'(u_+))>\delta_{\mathrm{fix}}$. Three runs on the window of one period ($D=1$)
fail it because $\delta_{\mathrm{fix}}$ is large there ($0.12$--$0.22$, against a pull of at most $0.04$): the target lies only
about five steps from zero frequency, so the mirror and constant terms are not small. Five runs on $D=4$ fail it because their phase is
close to quadrature ($|\cos(\varphi-\psi)|\le0.28$). Three runs on $D=4$ satisfy it, one only marginally (pull $0.0201$ against
$\delta_{\mathrm{fix}}=0.0198$) and two with an amplitude of only $0.02$ of the target's, which moves the line slowly while the
coefficients change. We restarted all eleven runs from their state at the start of joint training (under the original protocol all eleven arrive again) under
five variants: the original protocol (Adam, learning rate $10^{-3}$, minibatch $40$, joint); Adam on the full training split;
gradient descent on the full split (learning rate $0.02$, $2\times10^5$ steps), joint; the same with the ansatz parameters frozen,
so that the coefficients are fixed as in Proposition~\ref{prop:real}(b); and Adam on minibatches with the ansatz frozen.
\begin{center}\small
\begin{tabular}{lcccccc}
\toprule
run & condition & Adam, mb & Adam, full & GD, full & GD, frozen & Adam, mb, frozen\\
\midrule
$1.5$ steps, warm-up $500$ & marginal & $0.00$ & $0.00$ & $0.00$ & $0.03$ & $0.03$\\
$1.7$ steps, cold, side $-1$ & holds & $0.00$ & $0.00$ & $1.41$ & $1.38$ & $1.34$\\
$1.7$ steps, cold, side $+1$ & holds & $0.00$ & $0.00$ & $1.42$ & $1.45$ & $1.43$\\
\bottomrule
\end{tabular}
\end{center}
Final distance (steps) for the three runs that satisfy the condition ($D=4$, $\omega=5.3$). For the marginal run we repeated
gradient descent with frozen coefficients on all $400$ grid points and on $4000$ equally spaced points: the line then moves away
from the target (final distance $2.4$ steps), so its crossing is due to the random training split. In the two other runs the
condition fails within the first $70$ joint steps of the original run (the coefficient turns from $\cos(\varphi-\psi)=-0.99$ to
$-0.37$ while the line moves $0.2$ step) and holds again only after the line has passed the side lobe. Of the eight runs that do not
satisfy the condition at the start, three also arrive (final distance below $0.05$ step) with frozen coefficients under at least one optimiser.

\paragraph{Arrival and spacing.}\phantomsection\label{app:arrspacing} To check whether the spacing affects arrival beyond the start distance it sets, we fitted, per
arm, logistic regressions of arrival (final distance below $0.1$ step) on the start distance and the spacing (both linear) for the
$480$ target lines per arm on the grids with $g\ge1.5$, separately for lines that started within $0.75$ steps ($323$) and farther
($157$). Spacing coefficients (per step, 95\% Wald intervals): within $0.75$ steps $0.23$ $[-0.40,0.87]$ cold and $0.01$
$[-1.18,1.19]$ warm; farther $-0.86$ $[-1.48,-0.23]$ ($p=0.007$) cold and $-0.45$ $[-1.09,0.18]$ warm. Over all $480$ lines,
$-0.35$ $[-0.72,0.03]$ cold and $-0.33$ $[-0.84,0.18]$ warm. The dense grids are left out because there two model lines lie within
one step of every target line and arrival at one of them is not required for the fit (Section~\ref{sec:single-spacing}).

\paragraph{Longer training on dense grids.}\phantomsection\label{app:denselong} The cold and warm arms of the grid-spacing
experiment for $g=0.5$ and $g=1$ were rerun with three times the joint budget (same targets, splits, seeds and warm-up). Of the
target lines that had not arrived (final distance at least $0.1$ step), $47$--$68\%$ move more than $0.02$ step closer, and their
median distance falls from $0.15$--$0.19$ to $0.11$--$0.13$ steps; the share of arriving lines rises by $0$ to $8$ percentage points
($g=0.5$: $0.61\to0.61$ cold, $0.55\to0.58$ warm; $g=1$: $0.74\to0.78$ cold, $0.72\to0.80$ warm).

\paragraph{Learning rate and optimiser.}\phantomsection\label{app:lr} Setting of the warm-up study ($D=4$, $\omega=5.3$, random target phase, $16$ seeds,
starts $0.5$ to $1.9$ steps on either side), cold and warm arms, $7000$ steps in total, minibatch $40$. Share of runs that arrive
(final distance below $0.05$ step):
\begin{center}\small
\begin{tabular}{llccccccc}
\toprule
arm & optimiser & $0.5$ & $0.9$ & $1.1$ & $1.3$ & $1.5$ & $1.7$ & $1.9$\\
\midrule
cold & Adam $3\times10^{-4}$--$10^{-2}$ & $0.53$--$0.56$ & $0.44$--$0.47$ & $0.38$ & $0.34$--$0.38$ & $0.06$--$0.12$ & $0.06$ & $0$\\
cold & SGD $0.03$ & $0.44$ & $0.47$ & $0.38$ & $0.38$ & $0.06$ & $0.03$ & $0$\\
warm & Adam $3\times10^{-4}$ & $0.94$ & $0.69$ & $0.25$ & $0.12$ & $0$ & $0$ & $0$\\
warm & Adam $10^{-3}$ & $1.00$ & $0.88$ & $0.22$ & $0.03$ & $0$ & $0$ & $0$\\
warm & Adam $3\times10^{-3}$ & $1.00$ & $0.94$ & $0.03$ & $0$ & $0$ & $0$ & $0$\\
warm & Adam $10^{-2}$ & $1.00$ & $1.00$ & $0.03$ & $0$ & $0$ & $0$ & $0$\\
warm & SGD $0.03$ & $1.00$ & $1.00$ & $0.03$ & $0$ & $0$ & $0$ & $0$\\
\bottomrule
\end{tabular}
\end{center}
In cold runs the outcome is decided by the phase: $96\%$ of the starting configurations have the same outcome at every Adam
learning rate, and the lines that start inside the side lobe and do not arrive stop near it ($1.37$--$1.49$ steps, 5th to 95th
percentile). SGD with learning rates $0.1$ and $0.3$ is unstable and moves every line away (no arrivals).

\paragraph{Targets near zero frequency.}\phantomsection\label{app:nearzero} Window of one period ($D=1$, one step per frequency unit), target $\cos(\omega x+\psi)$ with
$\omega\in\{0.25,0.5,0.75,1,1.5,2,3\}$, start $0.3$ or $0.5$ steps on either side, $16$ phases, cold and warm arms as in
Section~\ref{sec:single}. Share that arrives (final $|\,|\alpha|-\omega|<0.05$): $0$ at $\omega=0.25$, $0.34$--$0.56$ at $0.5$,
$0.88$--$0.94$ at $0.75$ and $1$, $0.84$--$1.00$ from $1.5$ on. At $\omega=0.25$ the lines end at $|\alpha|=0.42$--$0.49$ (interquartile
range).

\paragraph{Shot noise.}\phantomsection\label{app:shots} Setting of the warm-up study, starts $0.3$ and $0.7$ steps on either side, $16$ seeds, four noise
realisations each. After $2000$ exact warm-up steps the ansatz is frozen and the prefactor is trained alone with Adam ($10^{-3}$,
$5000$ steps, all training points per step). Each step uses parameter-shift gradients for which every circuit evaluation (the
output and the two shifted outputs at each point) is estimated from $S$ shots. Arrival is the same for $S\in\{10^2,10^3,10^4,10^5\}$
as with exact gradients ($100\%$ from $0.3$ steps, $94\%$ from $0.7$ steps). The standard deviation of the prefactor over the last
$500$ steps is $0.0022$, $0.0008$, $0.0003$ and $0.0001$ steps (median), falling as $S^{-1/2}$. The ansatz is not trained under
noise here, so this isolates the final approach of the line.

\section{Several lines}
\label{app:coupling}

\paragraph{Velocity decomposition.}\phantomsection\label{app:velocity} By Eq.~\eqref{eq:dalpha}, the gradient with respect to $\alpha_j$ is a sum over model lines,
$\partial_{\alpha_j}\mathcal L=\sum_{\mathbf t}t_j\,F_{\mathbf t}$, where the sum runs over the lines with positive frequency and
$F_{\mathbf t}$ is the rate at which the loss changes when line $\mathbf t$ alone is moved (together with its mirror line at
$-\nu_{\mathbf t}$). Under gradient descent, line $\mathbf t_i$ therefore starts to move with velocity
\begin{equation}
  v_i=-\mathbf t_i\cdot\nabla_{\boldsymbol\alpha}\mathcal L
     =\underbrace{-|\mathbf t_i|^2F_{\mathbf t_i}}_{\text{own term}}\;-\sum_{\mathbf t\neq\mathbf t_i}(\mathbf t_i\cdot\mathbf t)\,F_{\mathbf t},
  \label{eq:mixing}
\end{equation}
its own term plus the contributions of every model line that shares a prefactor with it. The own term is what the single-line theory of
Section~\ref{sec:theory} describes; we call the sum of the other terms \emph{mixing}. For ternary prefactors $(g,3g,9g)$ every
digit pattern gives a distinct model line, so the decomposition of Eq.~\eqref{eq:mixing} can be computed exactly from the
circuit's Fourier coefficients. We computed it at the start of joint training for every run of the grid-spacing experiment of
Section~\ref{sec:single-spacing}; the reconstructed gradient matches the gradient of the simulated circuit to better than $5\times10^{-14}$.

\begin{figure}[h]\centering
\includegraphics[width=\linewidth]{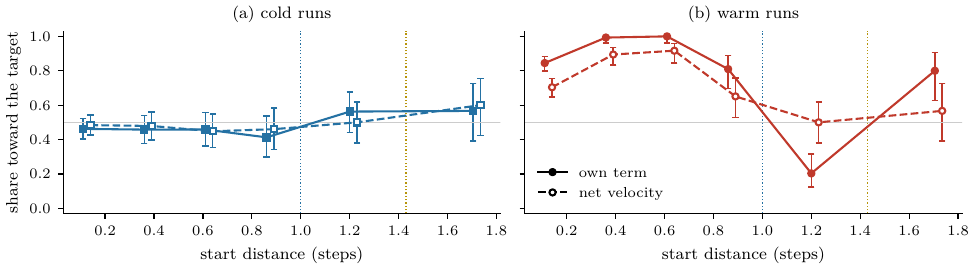}
\caption{Direction at the start of joint training (not arrival) in the grid-spacing experiment: share of target lines whose own
term (filled, solid) and net velocity (open, dashed) point toward the target, by start distance, for (a) cold and (b) warm runs
(95\% Wilson intervals). Dotted: one step and the first side lobe of a complex line. Arrival by start distance is shown in
Figure~\ref{fig:e7}b.}
\label{fig:mixing}
\end{figure}

\paragraph{Warm-up with several target lines.}\phantomsection\label{app:warmup3}
With one target line, the warm-up decided the outcome (Section~\ref{sec:single-line}). With three, warm lines start toward their
target more often (Figure~\ref{fig:mixing}), but the warm-up changes the share that arrives only on the separated grids
($g\ge1.5$): of the lines that started within $0.75$ steps, $98.8\%$ arrived warm and $94.7\%$ cold (paired by target line,
McNemar $p=0.004$); on the dense grids $64\%$ and $68\%$ ($p=0.38$). We repeated the start of joint training and evaluated the
same decomposition after $100$ joint steps. In cold runs, the share of own terms pointing toward the target rises in this period
from $29$--$52\%$ to $65$--$80\%$ (range over spacings, lines within one step), while the lines have covered between $-18\%$ (moving away) and
$7\%$ of their distance to the target (median per spacing); the amplitude of the nearest model line changes by a factor of
$0.91$--$1.28$. In warm runs the amplitudes change by at most $5\%$ in the same period. Consistent with this, the direction at the start of joint training does not predict arrival in cold
runs: lines within one step arrive in $83$--$84\%$ of cases whatever the direction of their own term and of their net velocity.
In warm runs it does: lines whose own term and net velocity both point toward the target arrive in $89\%$ of cases, lines whose
correct own term is reversed by mixing in $74\%$, and lines whose own term and net velocity both point away in $45\%$. With several
target lines, the warm-up therefore mainly changes which lines arrive, and on separated grids it raises the share slightly.

\begin{table}[h]\centering\small
\caption{Self versus target-driven part of the velocity at the start of joint training, grid-spacing experiment (prefactors
$(g,3g,9g)$, three target lines, $32$ targets per spacing), for the model line nearest each target line, target lines within $0.75$
steps: number of lines,
median ratio $|v_{\mathrm{self}}|/|v_{\mathrm{tgt}}|$, share of lines whose net direction differs from that of the target-driven
part, and share moving toward the target under the target-driven part and under the full gradient. The split reproduces the
prefactor gradient computed from the circuit to a relative error below $5\times10^{-14}$.}
\label{tab:selfsplit}
\begin{tabular}{llrrrrr}
\toprule
spacing (steps) & arm & lines & ratio & reversed & inward (tgt) & inward (full)\\
\midrule
$0.5$ & cold & $96$ & $0.33$ & $0.11$ & $0.53$ & $0.52$\\
 & warm & $96$ & $0.88$ & $0.25$ & $0.65$ & $0.58$\\
$1$ & cold & $96$ & $0.69$ & $0.16$ & $0.52$ & $0.47$\\
 & warm & $96$ & $0.62$ & $0.16$ & $0.78$ & $0.73$\\
$1.5$ & cold & $96$ & $0.24$ & $0.11$ & $0.48$ & $0.51$\\
 & warm & $96$ & $0.19$ & $0.07$ & $0.85$ & $0.86$\\
$2$ & cold & $71$ & $0.38$ & $0.11$ & $0.44$ & $0.46$\\
 & warm & $71$ & $0.16$ & $0.06$ & $0.92$ & $0.89$\\
$2.5$ & cold & $61$ & $0.20$ & $0.18$ & $0.56$ & $0.54$\\
 & warm & $61$ & $0.09$ & $0.03$ & $0.90$ & $0.87$\\
$3$ & cold & $51$ & $0.32$ & $0.10$ & $0.37$ & $0.35$\\
 & warm & $51$ & $0.12$ & $0.00$ & $0.90$ & $0.90$\\
$4$ & cold & $44$ & $0.24$ & $0.05$ & $0.39$ & $0.39$\\
 & warm & $44$ & $0.07$ & $0.02$ & $0.86$ & $0.89$\\
\bottomrule
\end{tabular}
\end{table}

\begin{table}[h]\centering\small
\caption{Start of joint training in the grid-spacing experiment: fraction of target lines whose own term, and whose net velocity,
point toward the target, by start distance ($277$, $142$, $96$, $63$, $64$ and $30$ lines per arm).}
\label{tab:mixing}
\begin{tabular}{lcccccc}
\toprule
start distance (steps) & $\le0.25$ & $0.25$--$0.5$ & $0.5$--$0.75$ & $0.75$--$1$ & $1$--$1.43$ & $1.43$--$2$\\
\midrule
own term, warm & 0.84 & 0.99 & 1.00 & 0.81 & 0.20 & 0.80\\
net velocity, warm & 0.70 & 0.89 & 0.92 & 0.65 & 0.50 & 0.57\\
own term, cold & 0.46 & 0.46 & 0.46 & 0.41 & 0.56 & 0.57\\
net velocity, cold & 0.48 & 0.48 & 0.45 & 0.46 & 0.50 & 0.60\\
\bottomrule
\end{tabular}
\end{table}

\begin{table}[h]\centering\small
\caption{Number of target lines $m$, drawn at random (distinct integers in the band plus an offset uniform in $[-\frac12,\frac12]$),
on the ternary circuits with $k=3$ (band $1$--$13$) and $k=4$ (band $1$--$40$): median held-out $R^2$, target lines placed by
training (started at least $0.1$ step away, ended within $0.1$ step; median per run), and mean distance to the nearest model line at
start and end (steps; median over runs). $16$ targets per cell. Same runs as
Figure~\ref{fig:complexity} for the values of $m$ shown there. For $m=24$, training places more target lines than the least-squares
prefactor change ($1$, median); a prefactor change of the trained size in a random direction places nearly as many
(Section~\ref{sec:coupling-nlines}).}
\label{tab:nlines}
\resizebox{\linewidth}{!}{%
\begin{tabular}{lcccccccc}
\toprule
$k=3$: $m$ & $1$ & $2$ & $3$ & $4$ & $6$ & $8$ & $10$ & $13$\\
\midrule
$R^2$ warm & $0.999$ & $0.998$ & $0.992$ & $0.978$ & $0.957$ & $0.932$ & $0.928$ & $0.922$\\
$R^2$ fixed & $0.939$ & $0.937$ & $0.914$ & $0.933$ & $0.911$ & $0.890$ & $0.898$ & $0.868$\\
lines placed & $1$ & $1$ & $2$ & $1.5$ & $2$ & $1$ & $2$ & $2$\\
distance start $\to$ end & $0.20\to0.00$ & $0.24\to0.03$ & $0.22\to0.06$ & $0.31\to0.13$ & $0.27\to0.20$ & $0.23\to0.20$ & $0.23\to0.22$ & $0.24\to0.23$\\
\midrule
$k=4$: $m$ & $1$ & $2$ & $4$ & $6$ & $8$ & $12$ & $16$ & $24$\\
\midrule
$R^2$ warm & $0.999$ & $0.997$ & $0.979$ & $0.918$ & $0.886$ & $0.854$ & $0.786$ & $0.699$\\
$R^2$ fixed & $0.920$ & $0.935$ & $0.904$ & $0.865$ & $0.803$ & $0.793$ & $0.737$ & $0.636$\\
lines placed & $1$ & $2$ & $2$ & $1$ & $1.5$ & $3$ & $2.5$ & $3.5$\\
distance start $\to$ end & $0.22\to0.00$ & $0.23\to0.01$ & $0.23\to0.08$ & $0.26\to0.16$ & $0.25\to0.16$ & $0.24\to0.20$ & $0.27\to0.25$ & $0.26\to0.24$\\
\bottomrule
\end{tabular}}
\end{table}

\paragraph{Placed lines and chance.}\phantomsection\label{app:placedchance} For the trained runs of Table~\ref{tab:nlines} with the values of $m$ of
Figure~\ref{fig:complexity}, each placed target line is attributed to its initially nearest model line (digit pattern
$\mathbf t_i$ of Section~\ref{sec:coupling-levers}) if that line is also its nearest at the end, and to another model line
otherwise. The chance reference applies to each run its own prefactor change with the components permuted and their signs flipped at
random ($300$ draws) and counts the lines that their initially nearest model lines then place. Warm arm, mean per run; shift: median
distance by which the prefactor change moves the initially nearest model lines (steps), trained and least squares:
\begin{center}\small
\begin{tabular}{lccccccc}
\toprule
 & \multicolumn{3}{c}{$k=3$} & \multicolumn{4}{c}{$k=4$}\\
\cmidrule(lr){2-4}\cmidrule(lr){5-8}
$m$ & $4$ & $8$ & $13$ & $4$ & $8$ & $16$ & $24$\\
\midrule
placed & $1.62$ & $1.75$ & $2.00$ & $2.00$ & $1.94$ & $2.44$ & $3.50$\\
\quad by the initially nearest line & $1.44$ & $1.44$ & $1.62$ & $1.81$ & $1.94$ & $2.25$ & $3.19$\\
\quad by another line & $0.19$ & $0.31$ & $0.38$ & $0.19$ & $0.00$ & $0.19$ & $0.31$\\
random direction, same size & $0.62$ & $1.06$ & $1.54$ & $0.34$ & $0.95$ & $1.89$ & $2.85$\\
shift, trained / least squares & $0.21$ / $0.28$ & $0.21$ / $0.11$ & $0.19$ / $0.10$ & $0.18$ / $0.23$ & $0.22$ / $0.15$ & $0.15$ / $0.12$ & $0.16$ / $0.09$\\
\bottomrule
\end{tabular}
\end{center}
In the cold arm, other model lines place more ($0.4$--$0.9$ per run for $m\ge4$), and the counts by the initially nearest line exceed
the random reference by at most $0.6$ for $m\ge13$.

\paragraph{Choice of encoding.}\phantomsection\label{app:encoding} A circuit can spend additional encoding gates on additional
prefactors. On an $n=2$ re-uploading circuit with $E\in\{4,6,8,14\}$ encoding gates and targets with eight lines in $1$--$13$, we
compared unary, binary and ternary series and two ternary series restarted to keep the reach near the band (Table~\ref{tab:enc}).
At a fixed number of gates the ansatz is the same for every encoding, and only the prefactors differ. The band-sized block $(1,3,9)$,
repeated, is then the best or tied-best encoding until unary reaches the band at $E=14$; encodings whose reach far exceeds the band
(binary, ternary and two ternary blocks at $E=14$) fail at every learning rate and budget tried (training $R^2$ $0.17$--$0.30$; see
robustness below). The repeated block also improves with every added gate, but each gate adds an encoding layer and $15$ ansatz
parameters, so this improvement cannot be attributed to the additional prefactors. The failure at large reach is confounded with two
other changes: the $2R+1$ Fourier degrees of freedom of a reach $R$ then far exceed the $45$ to $120$ ansatz parameters of this
circuit, and reaches above about $128$ exceed the Nyquist limit of the $256$ equidistant sample points, so that lines outside the
band alias onto lines inside it. With one encoding layer and full SU$(2^n)$ blocks (Section~\ref{sec:single-init}), a reach of $121$
for a band of $13$ cost almost nothing. These experiments measure the fit only, not the placement of lines.

\begin{table}[h]\centering\small
\caption{Encodings with $E$ gates on an $n=2$ circuit, eight target lines in $1$--$13$: median held-out $R^2$ (warm arm, $16$ runs
per cell), with the reach of the encoding in parentheses. The band is $13$; at $E=6$ the two ternary blocks and the repeated
$(1,3,9)$ coincide.}
\label{tab:enc}
\resizebox{\linewidth}{!}{%
\begin{tabular}{lccccc}
\toprule
$E$ & unary & binary & ternary & ternary, two blocks & $(1,3,9)$ repeated\\
\midrule
$4$  & $0.335$ ($4$) & $0.864$ ($15$) & $0.423$ ($40$) & $0.595$ ($8$) & $0.902$ ($14$)\\
$6$  & $0.458$ ($6$) & $0.689$ ($63$) & $0.396$ ($364$) & $0.920$ ($26$) & $0.920$ ($26$)\\
$8$  & $0.598$ ($8$) & $0.168$ ($255$) & $-0.342$ ($3280$) & $0.428$ ($80$) & $0.970$ ($30$)\\
$14$ & $0.980$ ($14$) & $-1.327$ ($16\,383$) & $-0.857$ ($2.4\times10^6$) & $-0.832$ ($2186$) & $0.977$ ($56$)\\
\bottomrule
\end{tabular}}
\end{table}

\paragraph{Fit against target complexity.}\phantomsection\label{app:complexity} The targets of
Section~\ref{sec:coupling-complexity} in two further variants: random amplitudes, uniform in $[0.2,1]$, and equal amplitudes with
Gaussian noise of standard deviation $0.1$ added to the target scaled to $[-1,1]$, on the training and the held-out points.
Least-squares references are computed as for equal amplitudes: on all lines of the starting grid and on all lines at the final
prefactors, fitted on the training points. The picture is the same as for equal amplitudes (Figures~\ref{fig:complexity-amp}
and~\ref{fig:complexity-noise}). With random amplitudes, least squares on the starting grid reaches a median held-out $R^2$ of
$0.96$--$0.99$, training improves the circuit by $0.03$--$0.09$, and the trained lines raise the least-squares value by at most
$0.03$. With noise, all values are lower, since the noise is not predictable (least squares on the starting grid $0.83$--$0.96$ against
the noisy held-out points); training improves the circuit by $0.04$--$0.09$, and the trained lines raise the least-squares value by
at most $0.03$.

\begin{figure}[h]\centering
\includegraphics[width=\linewidth]{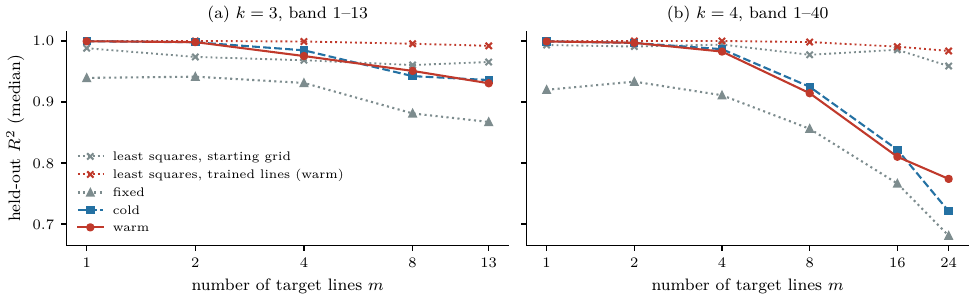}
\caption{As Figure~\ref{fig:complexity}, with amplitudes drawn uniformly from $[0.2,1]$.}
\label{fig:complexity-amp}
\end{figure}

\begin{figure}[h]\centering
\includegraphics[width=\linewidth]{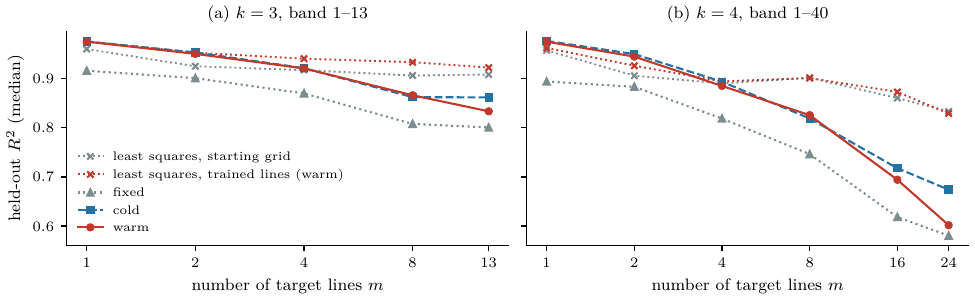}
\caption{As Figure~\ref{fig:complexity}, equal amplitudes with Gaussian noise of standard deviation $0.1$; $R^2$ against the noisy
held-out points.}
\label{fig:complexity-noise}
\end{figure}

\paragraph{Training budget.}\phantomsection\label{app:budget} The first eight equal-amplitude targets of
Section~\ref{sec:coupling-complexity} with $k=4$ and $m\in\{4,8,24\}$, fixed and warm arms, with $1$, $3$ and $10$ times the joint
budget ($1\times$: the runs of Section~\ref{sec:coupling-complexity}; the warm-up stays at $2000$ steps). Both arms get the same
total number of steps, $2000+5000f$ for a factor $f$ ($7000$, $17\,000$ and $52\,000$). Median held-out $R^2$, and
the shortfall of the training $R^2$ below least squares on the circuit's lines (an upper bound):
\begin{center}\small
\begin{tabular}{llcccccc}
\toprule
& & \multicolumn{3}{c}{held-out $R^2$} & \multicolumn{3}{c}{shortfall on training points}\\
\cmidrule(lr){3-5}\cmidrule(lr){6-8}
$m$ & arm & $1\times$ & $3\times$ & $10\times$ & $1\times$ & $3\times$ & $10\times$\\
\midrule
$4$ & fixed & $0.911$ & $0.938$ & $0.952$ & $0.056$ & $0.037$ & $0.028$\\
 & warm & $0.983$ & $0.993$ & $0.997$ & $0.011$ & $0.006$ & $0.002$\\
$8$ & fixed & $0.826$ & $0.887$ & $0.918$ & $0.100$ & $0.067$ & $0.049$\\
 & warm & $0.893$ & $0.934$ & $0.959$ & $0.069$ & $0.039$ & $0.026$\\
$24$ & fixed & $0.543$ & $0.675$ & $0.739$ & $0.205$ & $0.153$ & $0.122$\\
 & warm & $0.631$ & $0.748$ & $0.799$ & $0.181$ & $0.131$ & $0.096$\\
\bottomrule
\end{tabular}
\end{center}
Median paired gain of warm over fixed at $1\times$, $3\times$ and $10\times$: $+0.068$, $+0.055$, $+0.046$ ($m=4$); $+0.051$, $+0.062$,
$+0.043$ ($m=8$); $+0.079$, $+0.086$, $+0.038$ ($m=24$); warm is better on at least $7$ of $8$ targets in every case. Warm at
$1\times$ against fixed at $10\times$: $+0.036$ ($7$ of $8$), $-0.011$ ($2$ of $8$) and $-0.086$ ($0$ of $8$) for $m=4$, $8$ and $24$.

\paragraph{Mixing.} In Eq.~\eqref{eq:mixing}, the mixing term of line $i$ is split into the contributions of the model lines
nearest to the other two target lines and of all remaining model lines. A reversal is a line whose own term points toward its
target and whose net velocity points away; it is attributed to the group whose contribution has the larger magnitude ($50\%$ to the
other target lines, for the $158$ reversals within $0.75$ steps, both arms). The share of mixing is
$(|\text{mixing}_{\text{targets}}|+|\text{mixing}_{\text{rest}}|)/(|\text{own}|+|\text{mixing}_{\text{targets}}|+|\text{mixing}_{\text{rest}}|)$.
Arrival within one step by direction at the start of joint training: cold $0.83$--$0.84$ in every direction class; warm
$0.89$ (own term and net velocity toward the target, $433$ lines), $0.74$ (own term toward, reversed by mixing, $89$), $0.45$
(both away, $38$).

\paragraph{Designed geometries.}\phantomsection\label{app:geometries} Ternary $(1,3,9)$, three qubits. Each geometry places its target lines near the listed grid lines,
with offsets of size $d\in\{0.2,0.4\}$ steps and random signs $s_i$: $\boldsymbol\delta=d\,\mathbf s$ (one prefactor per line, shared prefactor, four lines), $(a,b,a+b)$ with $a=s_1d/2$, $b=s_2d/2$ (consistent), $(a,a,-2a)$ with $a=s_1d/2$ (conflicting) and $(-s_1d,s_1d)$ (opposite-sign sharing); and the amplitude of the last target line one third, one or three times that
of the others ($16$ targets per combination, pooled in Table~\ref{tab:levers}). The least-squares (LS) residual is
$|(\boldsymbol\delta-PP^+\boldsymbol\delta)_i|$, the distance at which the least-squares prefactor change leaves line $i$; it is zero
exactly when placement is possible.
\begin{table}[h]\centering\small
\caption{Designed geometries: median start distance, final distance (warm) and LS residual over all lines (steps), share of lines
starting toward their target, and median held-out $R^2$ (warm / fixed).}
\label{tab:levers}
\resizebox{\linewidth}{!}{\begin{tabular}{llcccccccc}
\toprule
geometry & grid lines & $\operatorname{rank}P$ & exact & start & final & LS residual & toward, cold & toward, warm & $R^2$ warm / fixed\\
\midrule
one prefactor per line & $1,3,9$ & $3$ & yes & $0.30$ & $0.084$ & $0$ & $0.54$ & $0.80$ & $0.995$ / $0.960$\\
shared prefactor & $3,4$ & $2$ & yes & $0.30$ & $0.046$ & $0$ & $0.50$ & $0.82$ & $0.999$ / $0.953$\\
dependent, consistent & $1,3,4$ & $2$ & yes & $0.10$ & $0.037$ & $0$ & $0.50$ & $0.88$ & $0.997$ / $0.975$\\
dependent, conflicting & $1,3,4$ & $2$ & no & $0.20$ & $0.144$ & $0.20$ & $0.51$ & $0.45$ & $0.995$ / $0.964$\\
four lines, two prefactors & $1,2,3,4$ & $2$ & no & $0.30$ & $0.169$ & $0.13$ & $0.51$ & $0.63$ & $0.990$ / $0.938$\\
opposite-sign sharing & $2,4$ & $2$ & yes & $0.30$ & $0.044$ & $0$ & $0.52$ & $0.94$ & $0.997$ / $0.920$\\
\bottomrule
\end{tabular}}
\end{table}
Consistent offsets satisfy $\delta_4=\delta_1+\delta_3$; conflicting offsets are $(a,a,-2a)$, which the shared prefactors cannot
produce, since moving the lines at $1$ and $3$ by $a$ moves the line at $4$ by $2a$. Dependence on amplitude, conflicting geometry,
median final distance per line (steps), with the prediction of least squares weighted by the squared amplitudes in parentheses:
\begin{center}\small
\begin{tabular}{lccc}
\toprule
amplitude of the line at $4$ & line at $1$ & line at $3$ & line at $4$\\
\midrule
$1/3$ & $0.09$ ($0.05$) & $0.11$ ($0.05$) & $0.32$ ($0.49$)\\
$1$ & $0.13$ ($0.20$) & $0.19$ ($0.20$) & $0.17$ ($0.20$)\\
$3$ & $0.15$ ($0.28$) & $0.29$ ($0.28$) & $0.07$ ($0.03$)\\
\bottomrule
\end{tabular}
\end{center}
In the conflicting geometry, the lines that start away from their target are mostly lines dragged toward the compromise by
stronger lines on shared prefactors.

\paragraph{Robustness.}\phantomsection\label{app:robust} Median held-out $R^2$ (warm, eight target lines in $1$--$13$) under other training settings; the
default is a learning rate of $10^{-3}$ and $7000$ steps (Table~\ref{tab:enc}).
\begin{center}\small
\begin{tabular}{llcccc}
\toprule
$E$ & encoding & default & lr $3\times10^{-4}$ & lr $3\times10^{-3}$ & $20\,000$ steps\\
\midrule
$8$ & unary & $0.598$ & $0.542$ & $0.603$ & $0.610$\\
 & binary & $0.168$ & $-0.514$ & $0.873$ & $0.877$\\
 & ternary & $-0.342$ & $-0.671$ & $0.489$ & $0.388$\\
 & ternary, two blocks & $0.428$ & $0.225$ & $0.628$ & $0.693$\\
 & $(1,3,9)$ repeated & $0.970$ & $0.868$ & $0.988$ & $0.992$\\
\midrule
$14$ & unary & $0.980$ & $0.802$ & $0.999$ & $1.000$\\
 & binary & $-1.327$ & $-1.058$ & $-1.217$ & $-1.134$\\
 & ternary & $-0.857$ & $-1.169$ & $-0.997$ & $-1.027$\\
 & ternary, two blocks & $-0.832$ & $-0.898$ & $-0.723$ & $-0.746$\\
 & $(1,3,9)$ repeated & $0.977$ & $0.916$ & $0.994$ & $0.995$\\
\bottomrule
\end{tabular}
\end{center}
At $E=8$, binary and ternary depend strongly on the settings, so these cells of Table~\ref{tab:enc} are under-trained at the
default. At $E=14$, the encodings whose reach far exceeds the band fail under every setting; their training $R^2$ at the default
is only $0.17$--$0.30$, so the failure is one of fitting, not of overfitting. The repeated $(1,3,9)$ is the best or tied-best
encoding at $E=8$ under every setting.

\section{Real data}
\label{app:realdata}

\paragraph{Protocol.}\phantomsection\label{app:registration} The protocol of Section~\ref{sec:realdata-protocol} was written before the
screen or any circuit was run; it is part of the code archive (\texttt{PROTOCOL\_rd\_v4.md}). It
replaced an earlier protocol in which the screen chose the band and selected the series on the same splits that later scored the
circuits. Its primary test could not reach significance with five resamples (Section~\ref{sec:realdata-protocol}); after the
five-resample runs we wrote an amendment (\texttt{PROTOCOL\_rd\_v4\_amendment1.md}, written before its runs) that adds $20$ fresh resamples of the minimal circuit, seeds $5$--$24$, and restricts the primary and secondary tests to them.
There were no other deviations.

\paragraph{Screen details.} Window: last $\min(n,512)$ observations; skipped if it contains missing values or is constant;
duplicates removed. Inputs $x_k=-\pi+2\pi k/n_w$. Target: affine trend removed, centred, scaled to $\max|y|=0.8$. Holdout: a random
$20\%$ of the window, permutation seeded by the first eight bytes of the SHA-256 hash of the series name. Band $W\in\{13,40\}$: the
smaller unless the larger improves the median inner test $R^2$ of the grid fit by at least $0.02$; capped so that
$2W+1\le0.6\,n_{\text{train}}$. Grid fit: least squares on $1,\cos wx,\sin wx$, $w=1,\dots,W$. Moved fit: the $K\in\{1,2,3\}$
strongest lines with $w\ge2$ get free frequencies within one step of their integer and not below $1.5$, fitted by variable
projection. Pass: $R^2_{\text{grid}}\ge0.6$, gain$_3\ge0.01$, the grid fit on the development set stays within $[-1,1]$, and the
strongest freed line is not pinned at $1.5$. The screen's numbers for all $49$ series are in \texttt{screen\_v4.csv} in the
code archive.

\paragraph{Series.} gas: Australian monthly gas production (\texttt{forecast}); gasoline: US weekly finished motor gasoline
product supplied (\texttt{fpp2}); sunspot: yearly sunspot numbers, $1700$--$1988$ (\texttt{datasets}); elec: US monthly
electricity production (\texttt{timeSeriesDataSets}); gasprice: US weekly gasoline prices (\texttt{quantreg}); BJsales: sales
series of Box and Jenkins (\texttt{datasets}); Nottingham: average monthly air temperatures at Nottingham Castle, $1920$--$1939$
(\texttt{datasets}). Training points per resample: $90$ (BJsales) to $308$; holdout: $30$ to $102$ points.

\paragraph{Circuit details.} Adam with learning rate $10^{-3}$, minibatch $40$, $7000$ steps per arm (warm: $2000$ steps with
frozen prefactors, then $5000$ joint steps). The repeated configuration has two encoding layers on the four qubits, each with
the prefactors of the minimal circuit (eight prefactors, $765$ ansatz parameters). Median held-out $R^2$ over $15$ runs, repeated configuration (pilot):
\begin{center}\small
\begin{tabular}{lccc}
\toprule
series & fixed & cold & warm\\
\midrule
gas & $0.898$ & $0.897$ & $0.904$\\
gasoline & $0.494$ & $0.490$ & $0.488$\\
sunspot & $0.675$ & $0.746$ & $0.669$\\
elec & $0.769$ & $0.776$ & $0.760$\\
gasprice & $0.977$ & $0.978$ & $0.980$\\
BJsales & $0.296$ & $0.280$ & $0.270$\\
Nottingham & $0.727$ & $0.680$ & $0.763$\\
\bottomrule
\end{tabular}
\end{center}

\paragraph{Warm versus cold and secondary comparisons.} Minimal circuit, $20$ confirmatory resamples: mean of the resample means of
the paired warm$-$cold difference, and the smallest difference a two-sided paired $t$-test over the resample means would detect at
the $5\%$ level with $80\%$ power, $(t_{0.975,19}+t_{0.8,19})\,s/\sqrt{20}$: gas $+0.003$ ($0.007$), gasoline $-0.005$ ($0.013$),
sunspot $-0.001$ ($0.024$), elec $-0.003$ ($0.007$). Of the $17$ secondary comparisons (cold$-$fixed and warm$-$cold for every series,
warm$-$fixed for the controls), four are significant after Holm correction: cold$-$fixed on gas ($p=0.003$), sunspot ($0.018$) and
elec ($0.0003$), and warm$-$fixed on gasprice ($0.0003$). All effect sizes, and the pilot's, are in \texttt{tests\_v4.csv}.

\paragraph{Pilot and amendment.}\phantomsection\label{app:pilot}
The protocol used five resamples. With five, the smallest attainable adjusted $p$ is $0.125$, so the primary test could
not reject at the $5\%$ level for any outcome. This was an error in our design, found after those runs. We wrote an amendment
that keeps everything else unchanged and adds $20$ fresh resamples of the minimal circuit for all seven series; the primary and
secondary tests use only these, and the five original resamples are reported as a pilot. The
repeated configuration was run only in the pilot. All resamples share the same development set and holdout, so the tests ask whether
a gain is robust to the choice of training subset and to training randomness on this holdout, not whether it holds for new data from
the same series. Given the development set, the resamples are drawn independently, so the sign test is exact conditional on it;
the overlap of the training subsets restricts the conclusion to this development set but does not invalidate the test.

\paragraph{Training budget (exploratory).}\phantomsection\label{app:rdbudget} Not part of the protocol; run after the confirmatory
analysis. On the confirmatory resamples $5$--$24$ with the same splits and ansatz seeds, the fixed and warm arms of the minimal circuit
were rerun with three times the joint budget ($2000+3\cdot5000=17\,000$ steps for both arms, as in
\appref{app:budget}{training budget}); at the original budget, the same code reproduces the confirmatory runs exactly (checked for the fixed arm on one series). Median over
resamples of the resample means; differences: mean over resamples, with the number of positive resamples and the one-sided sign test,
Holm-adjusted over the four selected series ($^\dagger$unadjusted for the other series).
\begin{center}\small
\resizebox{\linewidth}{!}{%
\begin{tabular}{lccccccc}
\toprule
series & fixed $1\times$ & fixed $3\times$ & warm $1\times$ & warm $3\times$ & warm $1\times-$fixed $3\times$ & warm $3\times-$fixed $3\times$\\
\midrule
gas        & $0.862$ & $0.871$ & $0.879$ & $0.893$ & $+0.009$, $15/20$, $p=0.06$ & $+0.027$, $18/20$, $p=0.0004$\\
gasoline   & $0.581$ & $0.583$ & $0.583$ & $0.577$ & $-0.005$, $7/20$, $p=0.94$ & $-0.006$, $9/20$, $p=0.75$\\
sunspot    & $0.764$ & $0.774$ & $0.803$ & $0.823$ & $+0.041$, $16/20$, $p=0.02$ & $+0.058$, $19/20$, $p<10^{-4}$\\
elec       & $0.848$ & $0.853$ & $0.858$ & $0.864$ & $+0.009$, $14/20$, $p=0.12$ & $+0.012$, $19/20$, $p<10^{-4}$\\
gasprice   & $0.942$ & $0.963$ & $0.953$ & $0.967$ & $-0.012$, $0/20$, $p=1^\dagger$ & $+0.005$, $19/20$, $p<10^{-4\,\dagger}$\\
BJsales    & $0.881$ & $0.894$ & $0.895$ & $0.905$ & $+0.008$, $13/20$, $p=0.13^\dagger$ & $+0.014$, $15/20$, $p=0.02^\dagger$\\
Nottingham & $0.912$ & $0.912$ & $0.910$ & $0.907$ & $+0.001$, $9/20$, $p=0.75^\dagger$ & $-0.004$, $6/20$, $p=0.98^\dagger$\\
\bottomrule
\end{tabular}}
\end{center}
Training three times as long raises the fixed arm by $0.000$--$0.021$ and the warm arm by up to $0.020$. Training the prefactors
remains ahead at equal budget on gas, sunspot, elec and gasprice; on gasprice, the screen negative, its advantage at the original
budget is smaller than what the fixed arm gains from more steps. We did not run ten times the budget.

\paragraph{Repeated circuit.}\phantomsection\label{app:repeated}
In the pilot, the repeated configuration, which fitted best among the encodings for noise-free synthetic targets (Appendix~\ref{app:coupling}, choice of encoding), overfits: its training
$R^2$ is higher than that of the minimal configuration on all seven series, and its held-out $R^2$ is lower on five of them (median
paired difference $-0.07$ to $-0.59$; higher by $0.02$ on gas and gasprice). Its gains from training are inconsistent across
resamples on every selected series (warm $-$ fixed $-0.003$ to $+0.007$), although its lines mostly reach the positions found by
least squares (final distance at most $0.02$ steps, median). With $90$ to $308$ noisy training points, the repeated configuration has
both more ansatz parameters and more prefactors than the minimal one; our experiments do not separate which of the two causes the
overfitting.

\begin{figure}[h]\centering
\includegraphics[width=\linewidth]{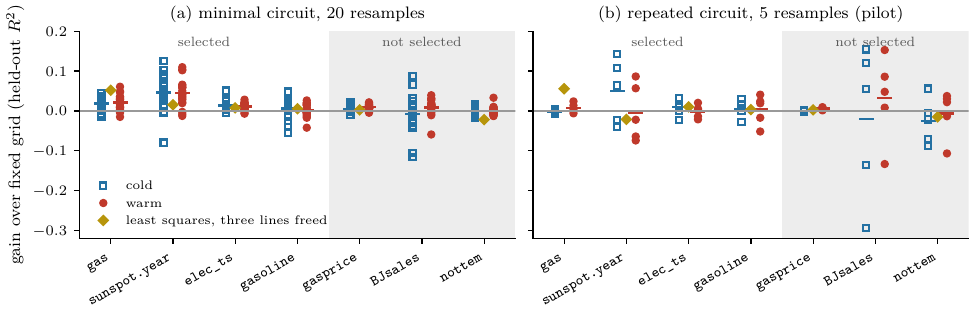}
\caption{Gain of training the prefactors over the fixed grid, held-out $R^2$: (a) minimal circuit, $20$ confirmatory resamples;
(b) repeated circuit, $5$ pilot resamples. Markers: resample means (paired over three ansatz seeds), cold (open squares) and warm
(filled circles); bars: their mean. Diamonds: least squares with three lines freed, gain over the integer grid on the holdout (median
over the same resamples; not shown for BJsales, where least squares on the grid fails). Shaded: series not selected by the screen.}
\label{fig:realdata}
\end{figure}

\paragraph{Sensitivity of the selection.}
An earlier version of the screen, which used all points and scored the circuits on the screen's own test splits, selected gas,
monthly US gas consumption and gasoline. On the development set (four fifths of the points), US gas no longer passes (its grid fit
falls below $0.6$) and sunspot and elec pass instead. With a few hundred points per series, which series the screen selects
depends on the sample, and we report the selected series as those the screen indicated, not as series known to contain off-grid
lines.

\paragraph{Line movement.} Median distance (steps) from the lines freed by least squares to the circuit's nearest model line,
before and after training (minimal, warm, confirmatory resamples): gas, freed lines at $2.03$, $39.70$, $39.70$: $0.03\to0.13$,
$0.30\to0.39$, $0.30\to0.39$; elec, $3.00$, $39.30$, $40.01$: $0.00\to0.09$, $0.30\to0.15$, $0.01\to0.04$; sunspot, $2.00$, $26.85$,
$30.00$: $0.00\to0.06$, $0.15\to0.20$, $0.00\to0.08$. In the repeated configuration the lines reach the freed positions (median final
distance at most $0.02$ steps, except $0.16$ for the warm arm on gas).

\section{Misspecified band and classical baselines}
\label{app:baselines}

Classical models with trainable frequencies are the learned counterpart of random Fourier
features~\citep{rahimi2007random,yang2015carte}. All runs use the targets, splits and run indices of the initialisation study (Section~\ref{sec:single-init}; $100$ runs per cell),
and a second target family, $\bar\omega_2=\{11,11.2,13\}$, the three-line example of \citet{Jaderberg_2024} shifted by $+10$ to the
top of the band.

\paragraph{Misspecified band.}\phantomsection\label{app:misspec} Ternary prefactors $(1,3,\dots,3^{k-1})$ on the circuit with one encoding gate per qubit and full
SU$(2^k)$ blocks, with $k$ set by a specified band edge $\omega_{\max}$; the true edge is $13$.
\begin{center}\small
\begin{tabular}{lcccccc}
\toprule
& & & \multicolumn{2}{c}{random off-grid lines} & \multicolumn{2}{c}{$\bar\omega_2$}\\
\cmidrule(lr){4-5}\cmidrule(lr){6-7}
$\omega_{\max}$ & $k$ & reach & fixed & warm & fixed & warm\\
\midrule
$4$ & $2$ & $4$ & 0.179 & 0.279 & $-0.078$ & $-0.089$\\
$13$ & $3$ & $13$ & 0.933 & 0.989 & 0.983 & 0.996\\
$40$ & $4$ & $40$ & 0.930 & 0.988 & 0.983 & 0.996\\
$121$ & $5$ & $121$ & 0.926 & 0.986 & 0.976 & 0.994\\
\bottomrule
\end{tabular}
\end{center}
Under-specification fails on every line outside the reach. Over-specification leaves the gain from training unchanged ($+0.045$ to
$+0.049$) and costs at most $0.001$ in paired median $R^2$ on random
lines; on $\bar\omega_2$ the edge at $121$ costs a small consistent $0.003$ (warm, $18$ of $20$ targets worse). The off-target share of
the output of trained circuits has a median below $0.5\%$.

\paragraph{Classical baselines.}\phantomsection\label{app:classical} Random Fourier features: $\phi(x)=\sqrt{2/D}\cos(Wx+b)$, $W\sim\mathcal N(0,\sigma^2)$ with $\sigma=8$ unless stated,
$b\sim U[0,2\pi)$, ridge regression ($10^{-3}$)~\citep{rahimi2007random,pedregosa2011scikit}. Coordinate MLP: two hidden layers of
$128$ $\tanh$ units (scikit-learn \texttt{MLPRegressor}: Adam with learning rate $10^{-3}$, default initialisation and minibatches of $200$, up to $5000$ epochs, stopping after $200$ epochs without improvement). Classical trainable frequencies: $f(x)=c_0+\sum_{j=1}^{13}a_j\cos\nu_jx+b_j\sin\nu_jx$ with $13$ independently
trainable frequencies, trained exactly as the circuits (Adam $10^{-3}$, minibatch $40$, cold and warm arms), started on the
integer grid $1,\dots,13$ or uniformly at random in $[0.5,13.5]$.
\begin{center}\small
\begin{tabular}{lcccc}
\toprule
& \multicolumn{2}{c}{random off-grid lines} & \multicolumn{2}{c}{$\bar\omega_2$}\\
\cmidrule(lr){2-3}\cmidrule(lr){4-5}
model & median & success & median & success\\
\midrule
quantum, grid, warm ($129$ parameters) & 0.989 & 84\% & 0.996 & 100\%\\
RFF, $D=64$ & 0.999 & 80\% & 0.960 & 55\%\\
RFF, $D=129$ & 1.000 & 99\% & 1.000 & 99\%\\
RFF, $D=129$, $\sigma=5$ & 0.978 & 59\% & 0.773 & 8\%\\
MLP & $-0.008$ & 10\% & $-0.067$ & 0\%\\
classical trainable frequencies, grid, warm ($40$ parameters) & 1.000 & 100\% & 1.000 & 100\%\\
classical trainable frequencies, random, warm & 1.000 & 70\% & 1.000 & 82\%\\
classical trainable frequencies, grid, fixed & 0.980 & 72\% & 0.992 & 98\%\\
\bottomrule
\end{tabular}
\end{center}
These baselines situate the results; they are not a comparison of quantum and classical models. On these noise-free targets,
random Fourier features with as many features as the circuit has parameters fit essentially perfectly, given a suitable bandwidth;
the coordinate MLP fails in most runs at these settings. The classical model with trainable frequencies reproduces the resolution
picture without any coupling: from the grid it succeeds in every run, from random starts in $70$--$82\%$, and its random starts fail
mainly where some target starts more than one step from every line (success $26$--$31\%$ against $98\%$ on the random off-grid
targets, and $50\%$ against $100\%$ on $\bar\omega_2$).

\end{document}